\documentclass{article}

\usepackage{amsmath,amsfonts,amsthm,amssymb,units,fullpage,hyperref,authblk}
\usepackage[round]{natbib}

\def\R{\mathbb{R}}
\def\XS{X_S}
\def\XL{X_L}
\def\DS{{\Delta_S}}
\def\DL{{\Delta_L}}
\def\Da{\Delta_{\operatorname{avg}}}
\def\Dm{\Delta_{\operatorname{mid}}}

\def\hXS{\hat{X}_S}
\def\hXL{\hat{X}_L}

\def\bXS{\bar{X}_S}
\def\bXL{\bar{X}_L}
\def\bO{{\bar{\Omega}}}
\def\bT{{\bar{T}}}
\def\bU{\bar{U}}
\def\bV{\bar{V}}
\def\bk{\bar{k}}
\def\br{\bar{r}}

\def\T{\mathcal{T}}
\def\I{\mathcal{I}}
\def\P{\mathcal{P}}
\newcommand{\ang}[1]{\langle #1 \rangle}
\renewcommand{\v}[1]{{\operatorname{vec}(#1)}}

\DeclareMathOperator{\tr}{tr}
\DeclareMathOperator{\rank}{rank}
\DeclareMathOperator{\range}{range}
\DeclareMathOperator{\supp}{supp}
\DeclareMathOperator{\sign}{sign}
\DeclareMathOperator{\orth}{orth}
\DeclareMathOperator{\clip}{clip}
\newtheorem{theorem}{Theorem}
\newtheorem{lemma}{Lemma}

\newtheorem{proposition}{Proposition}

\theoremstyle{remark}
\newtheorem{remark}{Remark}

\title{Robust Matrix Decomposition with Outliers}

\author[1,2]{Daniel Hsu}
\author[2]{Sham M.~Kakade}
\author[1]{Tong Zhang}
\affil[1]{Department of Statistics, Rutgers University}
\affil[2]{Department of Statistics, Wharton School, University of Pennsylvania}

\begin{document}

\maketitle

\begin{abstract}
Suppose a given observation matrix can be decomposed as the sum of a
low-rank matrix and a sparse matrix (outliers), 
and the goal is to recover these individual components from the observed sum.
Such additive decompositions have applications in a variety of numerical
problems including system identification, latent variable graphical
modeling, and principal components analysis.
We study conditions under which recovering such a decomposition is possible
via a combination of $\ell_1$ norm and trace norm minimization.
We are specifically interested in the question of how many outliers are
allowed so that convex programming can still achieve accurate recovery, and
we obtain stronger recovery guarantees than previous studies.
Moreover, we do not assume that the spatial pattern of outliers is random,
which stands in contrast to related analyses under such assumptions via
matrix completion.
\end{abstract}

\section{Introduction}

This work studies additive decompositions of matrices into sparse (outliers) and
low-rank components.
Such decompositions have found applications in a variety of numerical
problems, including system identification~\citep{CSPW09}, latent variable
graphical modeling~\citep{CPW10}, and principal component
analysis~\citep{CLMW09}.
In these settings, the user has an input matrix $Y \in \R^{m \times n}$
which is believed to be the sum of a sparse matrix $\XS$ and a low-rank
matrix $\XL$.
For instance, in the application to principal component analysis, $\XL$
represents a matrix of $m$ data points from a low-dimensional subspace of
$\R^n$, and is corrupted by a sparse matrix $\XS$ of errors before being
observed as
\[
\begin{array}{ccccc}
Y & = & \XS & + & \XL . \\
&
& \text{\scriptsize{(sparse)}}
&
& \text{\scriptsize{(low-rank)}}
\\
\end{array}
\]
The goal is to recover the original data matrix $\XL$ (and the error
components $\XS$) from the corrupted observations $Y$.
In the latent variable model application of~\citet{CPW10}, $Y$ represents
the precision matrix over visible nodes of a Gaussian graphical model, and
$\XS$ represents the precision matrix over the visible nodes when
conditioned on the hidden nodes.
In general, $Y$ may be dense as a result of dependencies between visible
nodes through the hidden nodes.
However, $\XS$ will be sparse when the visible nodes are mostly independent
after conditioning on the hidden nodes, and the difference $\XL = Y - \XS$
will be low-rank when the number of hidden nodes is small.
The goal is then to infer the relevant dependency structure from just the
visible nodes and measurements of their correlations.

Even if the matrix $Y$ is exactly the sum of a sparse matrix $\XS$ and a
low-rank matrix $\XL$, it may be impossible to identify these components
from the sum.
For instance, the sparse matrix $\XS$ may be low-rank, or the low-rank
matrix $\XL$ may be sparse.
In such cases, these components may be confused for each other, and thus
the desired decomposition of $Y$ may not be identifiable.
Therefore, one must impose conditions on the sparse and low-rank components
in order to guarantee their identifiability from $Y$.

We present sufficient conditions under which $\XS$ and $\XL$ are
identifiable from the sum $Y$.
Essentially, we require that $\XS$ not be too dense in any single row or
column, and that the singular vectors of $\XL$ not be too sparse.
The level of denseness and sparseness are considered jointly in the
conditions in order to obtain the weakest possible conditions.
Under a mild strengthening of the condition, we also show that $\XS$ and
$\XL$ can be recovered by solving certain convex programs, and that the
solution is robust under small perturbations of $Y$.
The first program we consider is
\[ \min \ \lambda \|\XS\|_\v1 + \|\XL\|_* \]
(subject to certain feasibility constraints such as $\|\XS+\XL-Y\| \leq
\epsilon$)
where $\|\cdot\|_\v1$ is the entry-wise $1$-norm and $\|\cdot\|_*$ is the
trace norm.
These norms are natural convex surrogates for the sparsity of $\XS$ and the
rank of $\XL$~\citep{T96,F02}, which are generally intractable to optimize.
We also considered a regularized formulation
\[ \min \ \frac1{2\mu} \|\XS+\XL-Y\|_\v2^2 + \lambda\|\XS\|_\v1 + \|\XL\|_*
\]
where $\|\cdot\|_\v2$ is the Frobenius norm;
such a formulation may be more suitable in certain applications and enjoys
different recovery guarantees.

\subsection{Related work}

Our work closely follows that of~\citet{CSPW09}, who initiated the study
of rank-sparsity incoherence and its application to matrix decompositions.
There, the authors identify parameters that
characterize the incoherence of $\XS$ and $\XL$ sufficient to guarantee
identifiability and recovery using convex programs.
However, their analysis of this characterization yields conditions that are
significantly stronger than those given in the our present work.
For instance, the allowed fraction of non-zero entries in $\XS$ is quickly
vanishing as a function of the matrix size, even under the most favorable
conditions on $\XL$; our analysis does not have this restriction and allows
$\XS$ to have up to $\Omega(mn)$ non-zero entries when $\XL$ is low-rank
and has non-sparse singular vectors.
Therefore, for instance, in the application to principal component
analysis, our analysis allows for up to a constant fraction of the data
matrix entries to be corrupted by noise of arbitrary magnitude, whereas the
analysis of~\citeauthor{CSPW09} requires that it decrease as a function of
the matrix dimensions.
Moreover, \citeauthor{CSPW09} only consider exact decompositions, which may
be unrealistic in certain applications; we allow for approximate
decompositions, and study the effect of perturbations on the accuracy of
the recovered components.

The application to principal component analysis with gross sparse errors
was studied by~\citet{CLMW09}, building on previous results and analysis
techniques for the related matrix completion problem
(\emph{e.g.},~\citealt{CR09,G09}).
The sparse errors model studied by~\citeauthor{CLMW09} requires that the
support of the sparse matrix $\XS$ be random, which can be unrealistic in
some settings.
However, the conditions are significantly weaker than those
of~\citet{CSPW09}: for instance, they allow for $\Omega(mn)$ non-zero
entries in $\XS$.
Our work makes no probabilistic assumption on the sparsity pattern of $\XS$
and instead studies purely deterministic structural conditions.
The price we pay, however, is roughly a factor of $\rank(\XL)$ in what is
allowed for the support size of $\XS$ (relative to the probabilistic
analysis of~\citeauthor{CLMW09}).
Narrowing this gap with alternative deterministic conditions is an
interesting open problem.
Follow-up work to~\citep{CLMW09} studies the robustness of the recovery
procedure~\citep{ZWLCM10}, as well as quantitatively weaker conditions on
$\XS$~\citep{GWLCM10}, but these works are only considered under the random
support model.
Our work is therefore largely complementary to these probabilistic
analyses.

\subsection{Outline}

We describe our main results in Section~\ref{section:results}.
In Section~\ref{section:prelim}, we review a number of technical tools such
as matrix and operator norms that are used to characterize the
rank-sparsity incoherence properties of the desired decomposition.
Section~\ref{section:incoherence} analyses these incoherence properties in
detail, giving sufficient conditions for identifiability as well as for
certifying the (approximate) optimality of a target decomposition for our
optimization formulations.
The main recovery guarantees are proved in Sections~\ref{section:opt1}
and~\ref{section:opt2}.

\section{Main results} \label{section:results}

Fix an observation matrix $Y \in \R^{m \times n}$.
Our goal is to (approximately) decompose the matrix $Y$ into the sum of a
sparse matrix $\XS$ and a low-rank matrix $\XL$.

\subsection{Optimization formulations}

We consider two convex optimization problems over $(\XS,\XL) \in \R^{m
\times n} \times \R^{m \times n}$.
The first is the constrained formulation (parametrized by $\lambda > 0$,
$\epsilon_\v1 \geq 0$, and $\epsilon_* \geq 0$)
\begin{equation} \label{eq:opt1}
\begin{array}{cl}
\min & \lambda \|\XS\|_\v1 + \|\XL\|_* \\
\text{s.t.} & \|\XS+\XL-Y\|_\v1 \leq \epsilon_\v1 \\
& \|\XS+\XL-Y\|_* \leq \epsilon_*
\end{array}
\end{equation}
where $\|\cdot\|_\v1$ is the entry-wise $1$-norm, and $\|\cdot\|_*$ is the
trace norm (\emph{i.e.}, sum of singular values).
The second is the regularized formulation (with regularization parameter
$\mu > 0$)
\begin{equation} \label{eq:opt2}
\begin{array}{cl}
\min & \frac1{2\mu} \|\XS+\XL-Y\|_\v2^2 + \lambda \|\XS\|_\v1 + \|\XL\|_*
\end{array}
\end{equation}
where $\|\cdot\|_\v2$ is the Frobenius norm (entry-wise $2$-norm).

We also consider adding a constraint to control $\|\XL\|_\v\infty$, the
entry-wise $\infty$-norm of $\XL$.
To~\eqref{eq:opt1}, we add the constraint
\[ \|\XL\|_\v\infty \leq b \]
and to~\eqref{eq:opt2}, we add
\[ \|\XS-Y\|_\v\infty \leq b \]
The parameter $b$ is intended as a natural bound for $\XL$ and is typically
known in applications.
For example, in image processing, the values of interest may lie in the
interval $[0,255]$, and hence, we may take $b = 500$ as a relaxation of the
box constraint $[0,255]$.
The core of our analyses do not rely on these additional constraints; we
only consider them to obtain improved robustness guarantees for recovering
$\XL$, which may be important in some applications.

\subsection{Identifiability conditions}

Our first result is a refinement of the \emph{rank-sparsity incoherence}
notion developed by~\citet{CSPW09}.
We characterize a target decomposition of $Y$ into $Y = \bXS + \bXL$ by the
projection operators to subspaces associated with $\bXS$ and $\bXL$.
Let
\[
\bO \ = \
\Omega(\bXS) \ := \
\{ X \in \R^{m \times n} : \supp(X) \subseteq \supp(\bXS) \}
\]
be the space of matrices whose supports are subsets of the support of
$\bXS$, and let $\P_{\bO}$ be the orthogonal projector to $\bO$ under the
inner product $\ang{A,B} = \tr(A^\top B)$, where $\P_\bO(M)$ is given by
\[ [\P_\bO(M)]_{i,j}
= \left\{ \begin{array}{cl}
M_{i,j} & \text{if $(i,j) \in \supp(\bXS)$} \\
0 & \text{otherwise}
\end{array} \right.
\ \forall i \in [m], j \in [n]
.
\]
Furthermore, let
\[
\bT \ = \
T(\bXL) \ := \
\{ X_1 + X_2 \in \R^{m \times n} : \range(X_1) \subseteq
\range(\bXL), \ \range(X_2^\top) \subseteq \range(\bXL^\top) \}
\]
be the span of matrices either in the row-space or column-space of $\bXL$,
and let $\P_{\bT}$ be the orthogonal projector to $\bT$, again, under the
inner product $\ang{A,B} = \tr(A^\top B)$; this is given by
\[ \P_\bT(M) = \bU\bU^\top M + M \bV\bV^\top - \bU\bU^\top M \bV\bV^\top
\]
where $\bU \in \R^{m \times \br}$ and $\bV \in \R^{n \times \br}$ are,
respectively, matrices of left and right orthonormal singular vectors
corresponding to the non-zero singular values of $\bXL$, and $\br$ is the
rank of $\bXL$.
We will see that certain operator norms of $\P_{\bO}$ and $\P_{\bT}$ can be
bounded in terms of structural properties of $\bXS$ and $\bXL$.
The first property measures the maximum number of non-zero entries in any
row or column of $\bXS$:
\begin{equation*}
\alpha(\rho) \ := \ \max\left\{ \rho \|\sign(\bXS)\|_{1\to1}, \
\rho^{-1} \|\sign(\bXS)\|_{\infty\to\infty} \right\}
\end{equation*}
where $\|M\|_{p\to q} := \max \{ \|Mv\|_q : v \in \R^n, \|v\|_p \leq 1 \}$,
\[
\sign(M)_{i,j} = \left\{ \begin{array}{rl}
-1 & \text{if $M_{i,j} < 0$} \\
0 & \text{if $M_{i,j} = 0$} \\
+1 & \text{if $M_{i,j} > 0$}
\end{array} \right.
\quad \forall i \in [m], j \in [n]
\]
and $\rho > 0$ is a balancing parameter to accommodate disparity between
the number of rows and columns; a natural choice for the balancing
parameter is $\rho := \sqrt{n/m}$.
We remark that $\rho$ is only a parameter for the analysis; the
optimization formulations do not directly involve $\rho$.
Note that $\bXS$ may have $\Omega(mn)$ non-zero entries and
$\alpha(\sqrt{n/m}) = O(\sqrt{mn})$ as long as the non-zero entries of
$\bXS$ are spread out over the entire matrix.
Conversely, a sparse matrix with just $O(m + n)$ could have
$\alpha(\sqrt{n/m}) = \sqrt{mn}$ by having all of its non-zero entries in
just a few rows and columns.

The second property measures the sparseness of the singular vectors of
$\bXL$:
\begin{equation*}
\beta(\rho) \ := \
\rho^{-1}\|\bU\bU^\top\|_\v\infty
+ \rho \|\bV\bV^\top\|_\v\infty
+ \|\bU\|_{2\to\infty} \|\bV\|_{2\to\infty}
.
\end{equation*}
For instance, if the singular vectors of $\bXL$ are perfectly aligned with
the coordinate axes, then $\beta(\rho) = \Omega(1)$.
On the other hand, if the left and right singular vectors have entries
bounded by $\sqrt{c/m}$ and $\sqrt{c/n}$, respectively, for some $c \geq
1$, then $\beta(\sqrt{n/m}) \leq 3c\br/\sqrt{mn}$.

Our main identifiability result is the following.
\begin{theorem} \label{theorem:identifiability}
If $\inf_{\rho > 0} \alpha(\rho) \beta(\rho) < 1$, then $\bO \cap \bT
= \{ 0 \}$.
\end{theorem}
Theorem~\ref{theorem:identifiability} is an immediate consequence of the
following lemma (also given as Lemma~\ref{lemma:dual-composition-bound}).
\begin{lemma} \label{lemma:composition}
For all $M \in \R^{m\times n}$,
$\|\P_{\bO}(\P_{\bT}(M))\|_\v1 \leq \inf_{\rho > 0} \alpha(\rho)
\beta(\rho) \|M\|_\v1$.
\end{lemma}
\begin{proof}[Proof of Theorem~\ref{theorem:identifiability}]
Take any $M \in \bO \cap \bT$.
By Lemma~\ref{lemma:composition}, $\|\P_{\bO}(\P_{\bT}(M))\|_\v1 \leq \alpha(\rho) \beta(\rho) \|M\|_\v1$.
On the other hand, $\P_{\bO}(\P_{\bT}(M)) = M$, so $\alpha(\rho)\beta(\rho)
< 1$ implies $\|M\|_\v1 = 0$, \emph{i.e.}, $M = 0$.
\end{proof}
Clearly, if $\bO \cap \bT$ contains a matrix other than $0$, then $\{ (\bXS
+ M, \bXL - M) : M \in \bO \cap \bT \}$ gives a family of sparse/low-rank
decompositions of $Y = \bXS + \bXL$ with at least the same sparsity and
rank as $(\bXS,\bXL)$.
Conversely, if $\bO \cap \bT = \{ 0 \}$, then any matrix in the direct sum
$\bO \oplus \bT$ has exactly one decomposition into a matrix $A \in \bO$
plus a matrix $B \in \bT$, and in this sense $(\bXS,\bXL)$ is identifiable.

Note that, as we have argued above, the condition $\inf_{\rho>0}
\alpha(\rho)\beta(\rho) < 1$ may be achieved even by matrices $\bXS$ with
$\Omega(mn)$ non-zero entries, provided that the non-zero entries of $\bXS$
are sufficiently spread out, and that $\bXL$ is low-rank and has singular
vectors far from the coordinate bases.
This is in contrast with the conditions studied by~\citet{CSPW09}.
Their analysis uses a different characterization of $\bXS$ and $\bXL$,
which leads to a stronger identifiability condition in certain cases.
Roughly, if $\bXS$ has an approximately symmetric sparsity pattern (so
$\|\sign(\bXS)\|_{1\to1} \approx \|\sign(\bXS)\|_{\infty\to\infty}$), then
\citeauthor{CSPW09} require $\alpha(1) \sqrt{\beta(1)} < 1$ for square $n
\times n$ matrices.\footnote{\citet{CSPW09} do not explicitly work out the
non-square case, but claim that $n$ can be replaced in their analysis by
the larger matrix dimension $\max\{m,n\}$.  However this does not seem
possible, and the analysis there should only lead to the quite  suboptimal
dimensionality dependency $\min\{m,n\}$.
This is because a rectangular matrix $\bXL$ will have left and right singular
vectors of different dimensions and thus different allowable ranges of
infinity norms.}
Since $\beta(1) = \Omega(1/n)$ for any $\bXL \in \R^{n \times n}$, the
condition implies $\alpha(1)^2 = O(n)$.
Therefore $\bXS$ must have at most $O(n)$ non-zero entries (or else
$\alpha(1)^2$ becomes super-linear).
In other words, the fraction of non-zero entries allowed in $\bXS$ by the
condition $\alpha(1) \sqrt{\beta(1)} < 1$ is quickly vanishing as a
function of $n$.

\subsection{Recovery guarantees}

Our next results are guarantees on (approximately) recovering the
sparse/low-rank decomposition $(\bXS,\bXL)$ from $Y = \bXS + \bXL$ via
solving either convex optimization problems~\eqref{eq:opt1}
or~\eqref{eq:opt2}.
We require a mild strengthening of the condition $\inf_{\rho>0}
\alpha(\rho)\beta(\rho) < 1$, as well as appropriate settings of $\lambda >
0$ and $\mu > 0$ for our recovery guarantees.
Before continuing, we first define another property of $\bXL$:
\[
\gamma := \ \|\bU\bV^\top\|_\v\infty
\]
which is approximately the same as (in fact, bounded above by) the third
term in the definition of $\beta(\rho)$.
The quantities $\alpha(\rho)$, $\beta(\rho)$, and $\gamma$ are central to our analysis. 
Therefore we state the following proposition for reference, which provides a more intuitive understanding of their behavior.
This is the only part in which explicit dimensional dependencies comes into our analysis.
\begin{proposition} \label{proposition:alpha-beta-gamma}
  Let $m_0$ the maximum number of non-zero entries of $\bXS$ per column and
  $n_0$ be the maximum number of non-zero entries of $\bXS$ per row. 
  Let $\bar{r}$ be the  rank of $\bU$ and $\bV$.
  Assume further that $m_0 \leq c_1 m/\bar{r}$ and $n_0 \leq c_1 n/\bar{r}$ for some $c_1 \in (0,1)$, and
  $\|\bU\|_\v\infty \leq \sqrt{c_2/m}$ and $\|\bV\|_\v\infty \leq \sqrt{c_2/n}$ for some $c_2 >0$.
  Then with $\rho=\sqrt{n/m}$, we have
  \[
  \alpha(\rho) \leq \frac{c_1}{\bar{r}}\sqrt{mn} ,
  \qquad \beta(\rho) \leq \frac{3 c_2 \bar{r}}{\sqrt{mn}} , 
  \qquad \gamma \leq  \frac{c_2 \bar{r}}{\sqrt{mn}} .
  \]
\end{proposition}

We now proceed  with conditions for the regularized formulation~\eqref{eq:opt2}.
Let $E := Y - (\bXS + \bXL)$ and
\begin{align*}
\epsilon_{2\to2} & \ := \ \|E\|_{2\to2} \\
\epsilon_\v\infty & \ := \ \|E\|_{\v\infty} + \|\P_\bT(E)\|_{\v\infty} .
\end{align*}
We require the following, for some $\rho > 0$ and $c > 1$:
\begin{align}
& \alpha(\rho) \beta(\rho) < 1
\label{eq:opt2-cond1} \\
& \lambda \ \leq \
\frac{(1 - \alpha(\rho)\beta(\rho))(1-c \cdot \mu^{-1} \epsilon_{2\to2})
- c \cdot \alpha(\rho) \mu^{-1} \epsilon_\v\infty
- c \cdot \alpha(\rho) \gamma}
{c \cdot \alpha(\rho)}
\label{eq:opt2-cond2} \\
& \lambda \ \geq \ c \cdot \frac{\gamma
+ \mu^{-1} 
 (2 - \alpha(\rho) \beta(\rho)) \epsilon_\v\infty}
{1 - \alpha(\rho) \beta(\rho) - c \cdot \alpha(\rho) \beta(\rho)}
\ > \ 0
.
\label{eq:opt2-cond3}
\end{align}
For instance, if for some $\rho > 0$,
\begin{equation} \label{eq:opt2-simple1}
\alpha(\rho) \gamma \leq \frac{1}{41}
\quad \text{and} \quad
\alpha(\rho) \beta(\rho) \leq \frac{3}{41}
,
\end{equation}
then the conditions are satisfied for $c = 2$ provided that $\mu$ and
$\lambda$ are chosen to satisfy
\begin{equation} \label{eq:opt2-simple2}
\mu \ \geq \ \max\left\{4 \cdot \epsilon_{2\to2}, \ \frac{2}{15} \cdot \frac{\epsilon_\v\infty}{\lambda}
\right\}
\quad \text{and} \quad
\frac{15}{2} \cdot \gamma
\ \leq \ \lambda \ \leq \
\frac{15}{82} \cdot \frac1{\alpha(\rho)}
.
\end{equation}
Note that (\ref{eq:opt2-simple1}) can be satisfied when $c_1 \leq c_2^{-1} /41$ in Proposition~\ref{proposition:alpha-beta-gamma}.

For the constrained formulation~\eqref{eq:opt1}, our analysis requires the
same conditions as above, except with $E$ set to $0$.
Note that our analysis still allows for approximate decompositions; it is
only the conditions that are formulated with $E = 0$.
Specifically, we require for some $\rho > 0$ and $c > 1$:
\begin{align}
& \alpha(\rho) \beta(\rho) < 1
\label{eq:opt1-cond1} \\
& \lambda \ \leq \
\frac{1 - \alpha(\rho)\beta(\rho)
- c \cdot \alpha(\rho) \gamma}
{c \cdot \alpha(\rho)}
\label{eq:opt1-cond2} \\
& \lambda \ \geq \ c \cdot \frac{\gamma}
{1 - \alpha(\rho) \beta(\rho) - c \cdot \alpha(\rho) \beta(\rho)}
\ > \ 0
.
\label{eq:opt1-cond3}
\end{align}
For instance, if for some $\rho > 0$,
\begin{equation} \label{eq:opt1-simple1}
\alpha(\rho) \gamma \leq \frac{1}{15}
\quad \text{and} \quad
\alpha(\rho) \beta(\rho) \leq \frac{1}{5}
,
\end{equation}
then the conditions are satisfied for $c = 2$ provided that
$\lambda$ is chosen to satisfy
\begin{equation} \label{eq:opt1-simple2}
5\gamma
\ \leq \ \lambda \ \leq \
\frac1{3\alpha(\rho)}
.
\end{equation}
Note that (\ref{eq:opt1-simple1}) can be satisfied when $c_1 \leq c_2^{-1} /15$ 
in Proposition~\ref{proposition:alpha-beta-gamma}.

In summary, Proposition~\ref{proposition:alpha-beta-gamma} shows that
our results can be applied even with $m_0=\Omega(m/\bar{r})$ and
$n_0=\Omega(n/\bar{r})$ outliers. In contrast, the results of \citet{CSPW09} only apply
under the condition $\max(m_0,n_0) = O(\sqrt{\min(m,n)/\bar{r}})$, which is significantly stronger.
Moreover, unlike the analysis of~\citet{CLMW09}, we do not have to assume
that $\supp(\bXS)$ is random. 

The following theorem gives our recovery guarantee for the constrained
formulation~\eqref{eq:opt1}.

\begin{theorem} \label{theorem:opt1}
Fix a target pair $(\bXS,\bXL) \in \R^{m \times n} \times \R^{m \times n}$
satisfying $\|Y - (\bXS+\bXL)\|_\v1 \leq \epsilon_\v1$ and $\|Y -
(\bXS+\bXL)\|_* \leq \epsilon_*$.
Assume the conditions~\eqref{eq:opt1-cond1}, \eqref{eq:opt1-cond2}, and
\eqref{eq:opt1-cond3} hold for some $\rho > 0$ and $c > 1$.
Let $(\hXS,\hXL) \in \R^{m \times n}$ be the solution to the convex
optimization problem~\eqref{eq:opt1}.
We have
\begin{multline*}
\max\left\{ \|\hXS-\bXS\|_\v1, \ \|\hXL-\bXL\|_\v1 \right\}
\\
\leq \
\left(
1 +
(1-1/c)^{-1} \cdot
\frac{2-\alpha(\rho)\beta(\rho)}{1-\alpha(\rho)\beta(\rho)}
\right)
\cdot \epsilon_\v1
+ (1-1/c)^{-1} \cdot
\frac{2-\alpha(\rho)\beta(\rho)}{1-\alpha(\rho)\beta(\rho)}
\cdot
\epsilon_* / \lambda
.
\end{multline*}
If, in addition for some $b \geq \|\bXL\|_\v\infty$, either:
\begin{itemize}
\item the optimization problem~\eqref{eq:opt1} is augmented with the
constraint $\|\XL\|_\v\infty \leq b$, or

\item $\hXL$ is post-processed by replacing $[\hXL]_{i,j}$ with
$\min\{ \max\{ [\hXL]_{i,j}, -b\}, b \}$ for all $i,j$,

\end{itemize}
then we also have
\[
\|\hXL-\bXL\|_\v2 \ \leq \ \min\left\{ \|\hXL-\bXL\|_\v1, \ \sqrt{2b
\cdot \|\hXL-\bXL\|_\v1}
\right\}
. \]
\end{theorem}
The proof of Theorem~\ref{theorem:opt1} is in Section~\ref{section:opt1}.
It is clear that if $Y = \bXS + \bXL$, then we can set $\epsilon_\v1 =
\epsilon_* = 0$ and we obtain exact recovery: $\hXS = \bXS$ and $\hXL =
\bXL$.
Moreover, any perturbation $Y - (\bXS + \bXL)$ affects the accuracy of
$(\hXS,\hXL)$ in entry-wise $1$-norm by an amount $O(\epsilon_\v1 +
\epsilon_* / \lambda)$.
Note that here, the parameter $\lambda$ serves to balance the entry-wise
$1$-norm and trace norm of the perturbation in the same way it is used in
the objective function of~\eqref{eq:opt1}.
So, for instance, if we have the simplified
conditions~\eqref{eq:opt1-simple1}, then we may choose $\lambda =
\sqrt{(5/3)\gamma/\alpha(\rho)}$ to satisfy \eqref{eq:opt1-simple2}, upon
which the error bound becomes
\[
\max\left\{ \|\hXS-\bXS\|_\v1, \ \|\hXL-\bXL\|_\v1 \right\}
\ = \
O\left(
\epsilon_\v1 + \sqrt{\frac{\alpha(\rho)}{\gamma}} \cdot \epsilon_*
\right)
.
\]
It is possible to modify the constraints in~\eqref{eq:opt1} to use norms
other than $\|\cdot\|_\v1$ and $\|\cdot\|_*$; the analysis could at the
very least be modified by simply using standard relationships to change
between norms, although this may introduce new slack in the bounds.
Finally, the second part of the theorem shows how the accuracy of $\hXL$
in Frobenius norm can be improved by adding an additional constraint or by
post-processing the solution.

Now we state our recovery guarantees for the regularized
formulation~\eqref{eq:opt2}.
\begin{theorem} \label{theorem:opt2}
Fix a target pair $(\bXS,\bXL) \in \R^{m \times n} \times \R^{m \times n}$.
Let $E := Y - (\bXS + \bXL)$ and
\begin{align*}
\epsilon_{2\to2} & \ := \ \|E\|_{2\to2} \\
\epsilon_\v\infty & \ := \ \|E\|_{\v\infty} + \|\P_\bT(E)\|_{\v\infty} \\
\epsilon_*' & \ := \ \|\P_\bT(E)\|_* .
\end{align*}
Let $\bk := |\supp(\bXS)|$ and $\br := \rank(\bXL)$.
Assume the conditions~\eqref{eq:opt2-cond1}, \eqref{eq:opt2-cond2}, and
\eqref{eq:opt2-cond3} hold for some $\rho > 0$ and $c > 1$.
Let $(\hXS,\hXL) \in \R^{m \times n}$ be the solution to the convex
optimization problem~\eqref{eq:opt2} augmented with the constraint
$\|\XS-Y\|_\v\infty \leq b$ for some $b \geq \|\bXS-Y\|_\v\infty$ ($b =
\infty$ is allowed).
Let
\begin{align*}
\br' \ := \
& \left(\lambda + \mu^{-1} \epsilon_\v\infty \right)
\cdot \frac{2\bk}{1-\alpha(\rho)\beta(\rho)}
\cdot \left( \lambda + \gamma + \mu^{-1} \epsilon_\v\infty \right)
\\
& {}
+ \left( 1 + 2\mu^{-1} \epsilon_{2\to2} \right)
\cdot 2\br
\cdot \left(
\frac{2\alpha(\rho)}{1-\alpha(\rho)\beta(\rho)}
\cdot \left( \lambda + \gamma + \mu^{-1} \epsilon_\v\infty \right)
+ 1 + 2\mu^{-1} \epsilon_{2\to2}
\right)
.
\end{align*}
We have
\begin{align*}
\|\hXS-\bXS\|_\v1
& \ \leq \
\frac
{
\br' \cdot (1-1/c)^{-1}\lambda^{-1} \cdot \mu
+ \lambda \bk \cdot \mu
+ 2\sqrt{\bk\br} \cdot \mu
+ \bk \cdot \epsilon_\v\infty
}
{1-\alpha(\rho)\beta(\rho)}
\\
\|\hXS-\bXS\|_\v2 & \ \leq \ \min\left\{ \|\hXS-\bXS\|_\v1, \
\sqrt{2b \cdot \|\hXS-\bXS\|_\v1}
\right\} \\
\|\hXL-\bXL\|_*
& \ \leq \
\sqrt{2\br} \cdot \|\hXS-\bXS\|_\v2
+ \epsilon_*'
+ \left( \frac{\br' \cdot (1-1/c)^{-1}}{2} + 2\br \right) \cdot \mu
.
\end{align*}
\end{theorem}
The proof of Theorem~\ref{theorem:opt2} is in Section~\ref{section:opt2}.
As before, if $Y = \bXS + \bXL$ so $E = 0$, then we can set $\mu \to 0$ and
obtain exact recovery with $\hXS = \bXS$ and $\hXL = \bXL$.
When the perturbation $E$ is non-zero, we control the accuracy of $\bXS$ in
entry-wise $1$-norm and $2$-norm, and the accuracy of $\bXL$ in trace norm.
Under the simplified conditions~\eqref{eq:opt2-simple1}, we can choose
$\lambda = (15/82)/\alpha(\rho)$ and $\mu =
\max\{4\epsilon_{2\to2}, 2\epsilon_\v\infty / (15\lambda)\}$ to satisfy
\eqref{eq:opt2-simple2}; this leads to the error bounds
\begin{align*}
& \|\hXS-\bXS\|_\v1
= O\left(
\br\alpha(\rho) 
\cdot \max\left\{\epsilon_{2\to2}, \ \alpha(\rho) \epsilon_\v\infty \right\}
\right)
\\
& \|\hXL-\bXL\|_*
= O\left(
\sqrt{\br} \cdot
\min\left\{ \sqrt{b \cdot \|\hXS-\bXS\|_\v1}, \
\|\hXS-\bXS\|_\v1 \right\}
+ \epsilon_*'
+ \br \cdot \max\left\{\epsilon_{2\to2}, \ \alpha(\rho) \epsilon_\v\infty
\right\}
\right)
\end{align*}
(here, we have used the facts
$\bk \leq \alpha(\rho)^2$,
$\alpha(\rho)\lambda = \Theta(1)$,
and $\br' = O(\br)$, which also implies that 
$\bk \cdot \epsilon_\v\infty= O(\alpha(\rho) \cdot \alpha(\rho) \epsilon_\v\infty)$).
Finally, note that if the constraint $\|\XS-Y\|_\v\infty \leq b$ is added
(\emph{i.e.}, $b < \infty$), then the requirement $b \geq
\|\bXS-Y\|_\v\infty$ can be satisfied with $b := \|\bXS\|_\v\infty +
\epsilon_\v\infty$.
This allows for a possibly improved bound on $\|\hXL-\bXL\|_*$.

\subsection{Examples}

We illustrate our main results with some simple examples.

\subsubsection{Random models}

We first consider a random model for the matrices $\bXS$ and
$\bXL$~\citep{CSPW09}.
Let the support of $\bXS$ be chosen uniformly at random $\tilde{k}$
times over the $[m] \times [n]$ matrix entries (so that one entry can be selected multiple times).
The value of the entries in the chosen support can be arbitrary.
With high probability, we have
\[
\|\sign(\bXS)\|_{1\to1} = O\left( \frac{\tilde{k}\log n}{n} \right) 
\quad \text{and} \quad
\|\sign(\bXS)\|_{\infty\to\infty} = O\left( \frac{\tilde{k}\log m}{m} \right) 
\]
so for $\rho := \sqrt{(n \log m)/(m \log n)}$, we have
\[
\alpha(\rho) = O\left( \tilde{k} \sqrt{\frac{(\log m)(\log n)}{mn}} \right)
.
\]
The logarithmic factors are due to collisions in the random process.
Now let $\bU$ and $\bV$ be chosen uniformly at random over all families of
$\br$ orthonormal vectors in $\R^m$ and $\R^n$, respectively.
Using arguments similar to those in~\citep{CR09}, one can show that with
high probability,
\begin{align*}
& \|\bU\bU^\top\|_\v\infty = O\left( \frac{\br \log m}{m} \right)
&
&
\|\bV\bV^\top\|_\v\infty = O\left( \frac{\br \log n}{n} \right)
\\
& \|\bU\|_{2\to\infty} = O\left( \sqrt{\frac{\br \log m}{m}} \right)
&
&
\|\bV\|_{2\to\infty} = O\left( \sqrt{\frac{\br \log n}{n}} \right)
,
\end{align*}
so for the previously chosen $\rho$, we have
\[
\beta(\rho) = O\left( \br \sqrt{\frac{(\log m)(\log n)}{mn}} \right)
\quad \text{and} \quad
\gamma = O\left( \br \sqrt{\frac{(\log m)(\log n)}{mn}} \right)
.
\]
Therefore
\[
\alpha(\rho)\beta(\rho) = O\left( \frac{\tilde{k} \br (\log m)(\log n)}{mn} \right)
\quad \text{and} \quad
\alpha(\rho)\gamma = O\left( \frac{\tilde{k} \br (\log m)(\log n)}{mn} \right)
,
\]
both of which are $\ll 1$ provided that
\[
\tilde{k} \leq \delta \cdot \frac{mn}{\br(\log m)(\log n)}
\]
for a small enough constant $\delta \in (0,1)$.
In other words, when $\bXL$ is low-rank, the matrix $\bXS$ can have nearly
a constant fraction of its entries be non-zero while still allowing for
exact decomposition of $Y = \bXS + \bXL$.
Our guarantee improves over that of~\citet{CSPW09} by roughly a factor of
$\Omega((mn)^{1/4})$, but is worse by a factor of $\br (\log m) (\log n)$
relative to the guarantees of~\citet{CLMW09} for the random model.
Therefore there is a gap between our generic deterministic analysis and a
direct probabilistic analysis of this random model, and this gap seems
unavoidable with sparsity conditions based on $\alpha(\rho)$.
It is an interesting open problem to find alternative characterizations of
$\supp(\bXS)$ that can narrow or close this gap.

\subsubsection{Principal component analysis with outliers}

Suppose $\bXL$ is matrix of $m$ data points lying in a low-dimensional
subspace of $\R^n$, and $Z$ is a random matrix with independent Gaussian
noise entries with variance $\sigma^2$.
Then $Y' = \bXL + Z$ is the standard model for principal component
analysis.
We augment the model with a sparse noise component $\bXS$ to obtain $Y =
\bXS + \bXL + Z$; here, we allow the non-zero entries of $\bXS$ to possibly
approach infinity.

According to Theorem~\ref{theorem:opt2}, we need to estimate
$\|Z\|_{2\to2}$, $\|Z\|_\v\infty$, $\|\P_\bT(Z)\|_\v\infty$, and
$\|\P_\bT(Z)\|_*$.
We have the following with high probability~\citep{DS01},
\[
\|Z\|_{2\to2} \leq \sigma\sqrt{m} + \sigma\sqrt{n} + O(\sigma)
.
\]
Using standard arguments with the rotational invariance of the Gaussian
distribution, we also have
\[
\|Z\|_\v\infty  \leq O(\sigma \log(mn))
\quad \text{and} \quad
\|\P_\bT(Z)\|_\v\infty \leq O(\sigma \log(mn))
\]
with high probability.
Finally, by Lemma~\ref{lemma:projections}, we have
\[
\|\P_\bT(Z)\|_*
\leq 2\br \|Z\|_{2\to2}
\leq 2\br\sigma\sqrt{m} + 2\br\sigma\sqrt{n} + O(\br\sigma)
.
\]
Suppose $(\bXS,\bXL)$ has $\alpha(\rho) \leq c_1 (\sqrt{mn}/\br)$,
$\beta(\rho) = \Theta(\br/\sqrt{mn})$, and $\gamma = \Theta(\br/\sqrt{mn})$
and satisfies the simplified condition~\eqref{eq:opt2-simple1}.
This can be achieved with $c_1 c_2 \leq 1/41$ in Proposition~\ref{proposition:alpha-beta-gamma}.
Also assume $\lambda$ and $\mu$ are chosen to
satisfy~\eqref{eq:opt2-simple2}, and that $b \geq \|\bXL\|_\v\infty +
\epsilon_\v\infty$.
Then we note that $\bar{k}= O(c_1^2 mn/\bar{r}^2)$, and thus have from Theorem~\ref{theorem:opt2}
(see the discussion thereafter):
\begin{align*}
\|\hXS-\bXS\|_\v1
& = O\left(
 c_1 \sqrt{mn}
\max\{ \sigma \sqrt{m} + \sigma \sqrt{n}, \ \sigma \sqrt{mn} \log(mn) / \br \}
\right)
\\
& = O\left(
\sigma c_1 mn \log(mn)/\br 
\right)
\\
\|\hXL-\bXL\|_*
& = O\left(
\sqrt{b \sigma c_1 mn \log(mn)/\bar{r}}
+ \bar{r} \sigma (\sqrt{m} + \sqrt{n})) 
+ c_1 \sqrt{mn} 
\right)
,
\end{align*}
where we may take $b=O(\sigma \log(mn)+\|\XL\|_\v\infty))$.

Now consider the situation where both $m, n \to \infty$, and assume that
$\|\bXL\|_\v\infty$ remains bounded. 
If $c_1 (\log(mn))^2 = o(1)$, which implies that the we have at most 
$o(m/(\log(mn))^2)$ outliers per column and $o(n/(\log(mn))^2)$ outliers per row,
then 
\[
\|\hXL-\bXL\|_* = O( \bar{r} \sigma (\sqrt{m} + \sqrt{n}) ) .
\]
That is, the normalized trace norm $\|\hXL-\bXL\|_*/\sqrt{nm} \to 0$. This means that we can correctly recover
the principal components of $\bXL$ with both outliers and random noise, when both $m$ and $n$ are large
and $c_1 (\log(mn))^2=o(1)$ in Proposition~\ref{proposition:alpha-beta-gamma}.

\section{Technical preliminaries} \label{section:prelim}

\subsection{Norms, inner products, and projections}

Our analysis involves a variety of norms of vectors, matrices (viewed as
elements of a vector space as well as linear operators of vectors), and
linear operators of matrices; we define these and related notions in this
section.

\subsubsection{Entry-wise norms}

For any $p \in [1,\infty]$, define $\|v\|_p := (\sum_i |v_i|^p)^{1/p}$ be
the $p$-norm of a vector $v$ (with $\|v\|_\infty := \max_i |v_i|$).
Also, define $\|M\|_\v{p} := (\sum_{i,j} |M_{i,j}|^p)^{1/p}$ to be the
entry-wise $p$-norm of a matrix $M$ (again, with $\|M\|_\v\infty :=
\max_{i,j} |M_{i,j}|$).
Note that $\|\cdot\|_\v2$ corresponds to the Frobenius norm.

\subsubsection{Inner products, linear operators, and orthogonal
projections}

We endow $\R^{m \times n}$ with the inner product $\ang{\cdot,\cdot}$
between matrices that induces the Frobenius norm $\|\cdot\|_\v2$; this is
given by $\ang{M,N} = \tr(M^\top N)$.

For a linear operator $\T : \R^{m \times n} \to \R^{m \times n}$, we denote
its adjoint by $\T^*$; this is the unique linear operator that satisfies
$\ang{\T^*(M),N} = \ang{M,\T(N)}$ for all $M \in \R^{m \times n}$ and
$N \in \R^{m \times n}$
(in this work, we only consider bounded linear operators).
For any two linear operators $\T_1$ and $\T_2$, we let $\T_1 \circ \T_2$
denote their composition as defined by $(\T_1 \circ \T_2)(M) :=
\T_1(\T_2(M))$.

Given a subspace $W \subseteq \R^{m \times n}$, we let $W^\perp$ denote its
orthogonal complement, and let $\P_W : \R^{m \times n} \to \R^{m \times n}$
denote the orthogonal projector to $W$ with respect to $\ang{\cdot,\cdot}$,
\emph{i.e.}, the unique linear operator with range $W$ and satisfying
${\P_W}^* = \P_W$ and $\P_W \circ \P_W = \P_W$.

\subsubsection{Induced norms}

For any two vector norms $\|\cdot\|_p$ and $\|\cdot\|_q$, define
$\|M\|_{p\to q} := \max_{x\neq0} \|Mx\|_q / \|x\|_p$ to be the corresponding
induced operator norm of a matrix $M$.
Our analysis uses the following special cases which have alternative
definitions:
\begin{itemize}
\item $\|M\|_{1\to1} = \max_j \|Me_j\|_1$,
\item $\|M\|_{1\to2} = \max_j \|Me_j\|_2$,
\item $\|M\|_{2\to2} = \text{spectral norm of $M$ (\emph{i.e.}, largest
singular value of $M$)}$,
\item $\|M\|_{2\to\infty} = \max_i \|M^\top e_i\|_2$, and
\item $\|M\|_{\infty\to\infty} = \max_i \|M^\top e_i\|_1$.
\end{itemize}
Here, $e_i$ is the $i$th coordinate vector which has a $1$ in the $i$th
position and $0$ elsewhere.

Finally, we also consider induced operator norms of linear matrix operators
$\T : \R^{m \times n} \to \R^{m \times n}$ (in particular, projection
operators with respect to $\ang{\cdot,\cdot}$).
For any two matrix norms $\|\cdot\|_{\diamondsuit}$ and
$\|\cdot\|_{\heartsuit}$, define $\|\T\|_{\diamondsuit\to\heartsuit} :=
\max_{M \neq 0} \|\T(M)\|_{\heartsuit} / \|M\|_{\diamondsuit}$.

\subsubsection{Other norms}

The trace norm (or nuclear norm) $\|M\|_*$ of a matrix $M$ is the sum of
the singular values of $M$.
We will also make use of a hybrid matrix norm $\|\cdot\|_{\sharp(\rho)}$,
parametrized by $\rho > 0$, which we define by
\[
\|M\|_{\sharp(\rho)}
:= \max\{ \rho\|M\|_{1\to1}, \ \rho^{-1} \|M\|_{\infty\to\infty} \}
.
\]
Also define $\|M\|_{\flat(\rho)} := \sup_{\|N\|_{\sharp(\rho)} \leq 1}
\ang{M,N}$, \emph{i.e.}, the dual of $\|\cdot\|_{\sharp(\rho)}$ (see
below).

\subsubsection{Dual pairs}

The matrix norm $\|\cdot\|_{\heartsuit}$ is said to be dual to
$\|\cdot\|_{\spadesuit}$ if, for all $M \in \R^{m \times n}$,
$\|M\|_{\heartsuit} = \sup_{\|N\|_{\spadesuit} \leq 1} \ang{M,N}$.

\begin{proposition} \label{proposition:holder}
Fix any matrix norm $\|\cdot\|_{\spadesuit}$, and let
$\|\cdot\|_{\heartsuit}$ be its dual.
For all $M \in \R^{m \times n}$ and $N \in \R^{m \times n}$, we have
\[
\ang{M,N} \leq \|M\|_{\spadesuit} \|N\|_{\heartsuit}
.
\]
\end{proposition}

\begin{proposition} \label{proposition:operator-dual}
Fix any any linear matrix operator $\T : \R^{m \times n} \to \R^{m \times
n}$ and any pair of matrix norms $\|\cdot\|_{\spadesuit}$ and
$\|\cdot\|_{\clubsuit}$.
We have
\[
\|\T\|_{\spadesuit\to\clubsuit}
\ = \
\|\T^*\|_{\diamondsuit\to\heartsuit}
,
\]
where $\|\cdot\|_{\heartsuit}$ is dual to $\|\cdot\|_{\spadesuit}$,
and $\|\cdot\|_{\diamondsuit}$ is dual to $\|\cdot\|_{\clubsuit}$.
\end{proposition}

The following pairs of matrix norms are dual to each other:
\begin{enumerate}
\item $\|\cdot\|_\v{p}$ and $\|\cdot\|_\v{q}$ where $1/p + 1/q = 1$;

\item $\|\cdot\|_*$ and $\|\cdot\|_{2\to2}$;

\item $\|\cdot\|_{\sharp(\rho)}$ and $\|\cdot\|_{\flat(\rho)}$ (by
definition).

\end{enumerate}

\subsubsection{Some lemmas}

First we show that the $\|\cdot\|_{\sharp(\rho)}$ norm (for any $\rho > 0$)
bounds the spectral norm $\|\cdot\|_{2\to2}$.
\begin{lemma} \label{lemma:spectral-bound}
For any $M \in \R^{m\times n}$, we have for all $\rho > 0$,
\[ 
\|M\|_{2\to2} \leq \|M\|_{\sharp(\rho)} .
\]
\end{lemma}
\begin{proof}
Let $\sigma$ be the largest singular value of $M$, and let $u \in \R^m$ and
$v \in \R^n$ be, respectively, associated left and right singular vectors.
Then
\[
\left\| \left[ \begin{array}{cc}
0 & \rho M \\
\rho^{-1} M^\top & 0
\end{array} \right]
\left[ \begin{array}{c}
\rho^{1/2} u \\
\rho^{-1/2} v
\end{array} \right]
\right\|_1
=
\left\|
\left[ \begin{array}{c}
\rho^{1/2} Mv \\
\rho^{-1/2} M^\top u
\end{array} \right]
\right\|_1
=
\sigma \left\|
\left[ \begin{array}{c}
\rho^{1/2} u \\
\rho^{-1/2} v
\end{array} \right]
\right\|_1
.
\]
Moreover, by definition of $\|\cdot\|_{1\to1}$,
\[
\left\| \left[ \begin{array}{cc}
0 & \rho M \\
\rho^{-1/2} M^\top & 0
\end{array} \right]
\left[ \begin{array}{c}
\rho^{1/2} u \\
\rho^{-1/2} v
\end{array} \right]
\right\|_1
\leq
\left\| \left[ \begin{array}{cc}
0 & \rho M \\
\rho^{-1} M^\top & 0
\end{array} \right]
\right\|_{1\to1}
\left\|
\left[ \begin{array}{c}
\rho^{1/2} u \\
\rho^{-1/2} v
\end{array} \right]
\right\|_1
.
\]
Therefore
\begin{align*}
\|M\|_{2\to2}
\ = \sigma
& \leq \left\| \left[ \begin{array}{cc}
0 & \rho M \\
\rho^{-1} M^\top & 0
\end{array} \right]
\right\|_{1\to1} \\
& = \max\{ \|\rho^{-1} M^\top\|_{1\to1}, \|\rho M\|_{1\to1}\}
\ = \max\{ \rho^{-1} \|M\|_{\infty\to\infty}, \rho \|M\|_{1\to1}\}
\ = \|M\|_{\sharp(\rho)}
.
\qedhere
\end{align*}
\end{proof}

The following lemma is the dual of Lemma~\ref{lemma:spectral-bound}.

\begin{lemma} \label{lemma:trace-bound}
For any $M \in \R^{m\times n}$, we have for all $\rho > 0$,
\[
\|M\|_{\flat(\rho)} \leq \|M\|_* .
\]
\end{lemma}
\begin{proof}
We know that $\|M\|_{\flat(\rho)}=\ang{M,N}$ for some matrix $N$ such that
$\|N\|_{\sharp(\rho)}=1$.
Therefore $\|N\|_{2\to2} \leq 1$ from Lemma~\ref{lemma:spectral-bound}, and
thus using Proposition~\ref{proposition:holder},
\[
\|M\|_{\flat(\rho)}= \ang{M,N} \leq \|M\|_{*}\|N\|_{2\to2} \leq \|M\|_* .
\qedhere
\]
\end{proof}

Finally we state a lemma concerning the invertibility of a certain block-form
operator used in our analysis.
\begin{lemma} \label{lemma:block-inverse}
Fix any matrix norm $\|\cdot\|_\spadesuit$ on $\R^{m \times n}$ and linear
operators $\T_1 : \R^{m \times n} \to \R^{m \times n}$ and $\T_2 : \R^{m
\times n} \to \R^{m \times n}$.
Let $\I : \R^{m \times n} \to \R^{m \times n}$ be the identity operator,
and suppose $\|\T_1 \circ \T_2\|_{\spadesuit\to\spadesuit} < 1$.
\begin{enumerate}
\item $\I - \T_1 \circ \T_2$ is invertible and satisfies
\[
\|(\I - \T_1 \circ \T_2)^{-1}\|_{\spadesuit\to\spadesuit}
\leq
\frac{1}{1 - \|\T_1 \circ \T_2\|_{\spadesuit\to\spadesuit}}
. \]

\item The linear operator on $\R^{m \times n} \times \R^{m \times n}$
\[
\left[ \begin{array}{cc}
\I & \T_1 \\
\T_2 & \I
\end{array} \right]
\]
is invertible, and its inverse is given by
\begin{align*}
\left[ \begin{array}{cc}
\I & \T_1 \\
\T_2 & \I
\end{array} \right]^{-1}
& =
\left[ \begin{array}{cc}
(\I - \T_1 \circ \T_2)^{-1}
& -(\I - \T_1 \circ \T_2)^{-1} \circ \T_1
\\
-\T_2 \circ (\I - \T_1 \circ \T_2)^{-1}
& \I + \T_2 \circ (\I - \T_1 \circ \T_2)^{-1} \circ \T_1
\end{array} \right]
\\
& =
\left[ \begin{array}{cc}
(\I - \T_1 \circ \T_2)^{-1}
& -(\I - \T_1 \circ \T_2)^{-1} \circ \T_1
\\
-(\I - \T_2 \circ \T_1)^{-1} \circ \T_2
& (\I - \T_2 \circ \T_1)^{-1}
\end{array} \right]
.
\end{align*}

\end{enumerate}
\end{lemma}
\begin{proof}
The first claim is a standard application of Taylor expansions.
The second claim then follows from formulae of block matrix inverses using
Schur complements.
\end{proof}

\subsection{Projection operators and subdifferential sets}

Recall the definitions of the following subspaces
\begin{equation*}
\Omega(\XS) := \{ X \in \R^{m \times n} : \supp(X) \subseteq
\supp(\XS) \}
\end{equation*}
and
\begin{equation*}
T(\XL) := \{ X_1 + X_2 \in \R^{m \times n} : \range(X_1) \subseteq
\range(\XL), \range(X_2^\top) \subseteq \range(\XL^\top) \}
.
\end{equation*}
The orthogonal projectors to these spaces are given in the following
proposition.
\begin{proposition} \label{proposition:projections}
Fix any $\XS \in \R^{m \times n}$ and $\XL \in \R^{m \times n}$.
For any matrix $M \in \R^{m \times n}$,
\[ [\P_{\Omega(X_S)}(M)]_{i,j}
= \left\{ \begin{array}{cl}
M_{i,j} & \text{if $(i,j) \in \supp(X_S)$} \\
0 & \text{otherwise}
\end{array} \right.
\]
for all $1 \leq i \leq m$ and $1 \leq j \leq n$,
and
\[ \P_{T(\XL)}(M) = UU^\top M + M VV^\top - UU^\top M VV^\top
\]
where $U$ and $V$ are the matrices of left and right singular vectors of
$X_L$.
\end{proposition}
\begin{lemma} \label{lemma:projections}
Under the setting of Proposition~\ref{proposition:projections},
\begin{align*}
\|\P_{\Omega(\XS)}(M)\|_\v1 & \leq \sqrt{|\supp(\XS)|}
\|\P_{\Omega(\XS)}(M)\|_\v2 \leq \sqrt{|\supp(\XS)|} \|M\|_\v2 \\
\|\P_{\Omega(\XS)}(M)\|_\v1 & \leq |\supp(\XS)|
\|\P_{\Omega(\XS)}(M)\|_\v\infty \leq |\supp(\XS)| \|M\|_\v\infty \\
\|\P_{T(\XL)}(M)\|_{2\to2} & \leq 2 \|M\|_{2\to2} \\
\|\P_{T(\XL)}(M)\|_* & \leq 2 \rank(\XL) \|M\|_{2\to2} \\
\|\P_{T(\XL)}(M)\|_\v2 & \leq 2\sqrt{\rank(\XL)} \|M\|_{2\to2}
.
\end{align*}
\end{lemma}
\begin{proof}
The first and second claims rely on the fact that
$|\supp(\P_{\Omega(\XS)}(M))| \leq |\supp(\XS)|$, as well as the fact that
$\P_{\Omega(\XS)}$ is an orthonormal projector with respect to the inner
product that induces the $\|\cdot\|_\v2$ norm.
For the third claim, note that
\[ \|\P_{T(\XL)}(M)\|_{2\to2}
\leq \|UU^\top M\|_{2\to2} + \|(I-UU^\top)MVV^\top\|_{2\to2}
\leq 2\|M\|_{2\to2}
. \]
The remaining claims use a similar decomposition as the third claim as well
as the fact that
\[ \max\{\rank(UU^\top M),\rank((I-UU^\top)MVV^\top)\} \leq \rank(\XL) .
\qedhere
\]
\end{proof}

Define 
\[
\sign(\XS) \in \{-1,0,+1\}^{m\times n}
\]
to be the matrix whose $(i,j)$th entry is $\sign([\XS]_{i,j})$, 
and define 
\[
\orth(\XL) := UV^\top ,
\]
where $U$ and $V$, respectively, are
matrices of the left and right orthonormal singular vectors of $\XL$
corresponding to non-zero singular values.
The following proposition characterizes the subdifferential sets for the
non-smooth norms $\|\cdot\|_\v1$ and $\|\cdot\|_*$~\citep{W92}.
\begin{proposition} \label{proposition:subgradient}
The subdifferential set of $\XS \mapsto \|\XS\|_\v1$ is
\begin{equation*}
\partial_{\XS} (\|\XS\|_\v1)
= \{ G \in \R^{m \times n} : \|G\|_\v\infty \leq 1, \P_{\Omega(\XS)}(G)
= \sign(\XS) \}
;
\end{equation*}
the subdifferential set of $\XL \mapsto \|\XL\|_*$ is
\begin{equation*}
\partial_{\XL} (\|\XL\|_*)
= \{ G \in \R^{m \times n} : \|G\|_{2\to2} \leq 1, \P_{T(\XL)}(G) =
\orth(\XL)
\}
.
\end{equation*}
\end{proposition}

The following lemma is a simple consequence of subgradient properties.
\begin{lemma} \label{lemma:subgradient}
Fix $\lambda > 0$ and define the function $g(\XS,\XL) := \lambda
\|\XS\|_\v1 + \|\XL\|_*$.
Consider any $(\bXS,\bXL)$ in $\R^{m \times n} \times \R^{m \times n}$.
If there exists $Q \in \R^{m \times n}$ such that: $Q$ is a subgradient of
$\lambda \|\XS\|_\v1$ at $\XS=\bXS$, $Q$ is a subgradient of $\|\XL\|_*$ at
$\XL=\bXL$, and $\|\P_{\Omega(\bXS)^\perp}(Q)\|_\v\infty \leq \lambda / c$ and
$\|\P_{T(\bXL)^\perp}(Q)\|_{2\to2} \leq 1 / c$ for some $c > 1$, then
\[
g(\XS,\XL) - g(\bXS,\bXL)
\ \geq \
\ang{Q,\XS + \XL - \bXS - \bXL}
+ (1-1/c) \left( \lambda \|\P_{\bO^\perp}(\XS - \bXS)\|_\v1
+ \|\P_{\bT^\perp}(\XL - \bXL)\|_* \right)
\]
for all $(X_S,X_L) \in \R^{m \times n} \times \R^{m \times n}$.
\end{lemma}
\begin{proof}
Let $\bO := \Omega(\bXS)$, $\bT := T(\bXL)$, $\DS := \XS - \bXS$, and $\DL
: \XL - \bXL$.
For any subgradient $G \in \partial_{\XS}(\lambda\|\bXS\|_\v1)$, we have $G
- Q = \P_{\bO}(G) + \P_{\bO^\perp}(G) - \P_{\bO}(Q) - \P_{\bO^\perp}(Q) =
\P_{\bO^\top}(G) - \P_{\bO^\top}(Q)$.
Therefore
\begin{align*}
\lefteqn{
\lambda \|\bXS + \DS\|_\v1 - \lambda \|\bXS\|_\v1 - \ang{Q,\DS}
} \\
& \geq \sup \{ \ang{G,\DS}
- \ang{Q,\DS} : G \in \partial_{\XS}(\lambda \|\bXS\|_\v1) \}\\
& \geq \sup \{ \ang{G-Q,\DS} : G \in \partial_{\XS}(\lambda \|\bXS\|_\v1) \} \\
& = \sup \{ \ang{\P_{\bO^\perp}(G) - \P_{\bO^\perp}(Q),\DS}
: G \in \partial_{\XS}(\lambda \|\bXS\|_\v1) \} \\
& = \sup \{ \ang{\P_{\bO^\perp}(G) - \P_{\bO^\perp}(Q),\P_{\bO^\perp}(\DS)}
: G \in \partial_{\XS}(\lambda \|\bXS\|_\v1) \} \\
& = \sup \{ \ang{\P_{\bO^\perp}(G),\P_{\bO^\perp}(\DS)}
- \ang{\P_{\bO^\perp}(Q),\P_{\bO^\perp}(\DS)}
: G \in \partial_{\XS}(\lambda \|\bXS\|_\v1) \} \\
& = \lambda  \|\P_{\bO^\perp}(\DS)\|_\v1
- \ang{\P_{\bO^\perp}(Q),\P_{\bO^\perp}(\DS)} \\
& \geq \lambda  \|\P_{\bO^\perp}(\DS)\|_\v1
- \|\P_{\bO^\perp}(Q)\|_\v\infty \|\P_{\bO^\perp}(\DS)\|_\v1 \\
& \geq \lambda  (1-1/c) \|\P_{\bO^\perp}(\DS)\|_\v1
\end{align*}
where the second-to-last inequality uses the duality of $\|\cdot\|_\v1$ and
$\|\cdot\|_\v\infty$ and Proposition~\ref{proposition:operator-dual}.
Similarly,
\begin{align*}
\|\bXL - \DL\|_* - \|\bXL\|_* -\ang{Q,\DL}
& \geq (1-1/c) \|\P_{\bT^\perp}(\DL)\|_*
\end{align*}
by noting the duality of $\|\cdot\|_*$ and $\|\cdot\|_{2\to2}$.
Combining these gives the desired inequality.
\end{proof}

\section{Rank-sparsity incoherence} \label{section:incoherence}

Throughout this section, we fix a target $(\bXS,\bXL) \in \R^{m \times n}
\times \R^{m \times n}$, and let $\bO := \Omega(\bXS)$ and $\bT :=
T(\bXL)$.
Also let $\bU$ and $\bV$ be, respectively, matrices of the left and right
singular vectors of $\bXL$ corresponding to non-zero singular values.
Recall the following structural properties of $\bXS$ and $\bXL$:
\begin{align*}
\alpha(\rho) & \ := \ \|\sign(\bXS)\|_{\sharp(\rho)}
\ = \ \max\{ \rho \|\sign(\bXS)\|_{1\to1}, \ \rho^{-1}
\|\sign(\bXS)\|_{\infty\to\infty} \}
; \\
\beta(\rho) &
\ := \ \rho^{-1}\|\bU\bU^\top\|_\v\infty
+ \rho \|\bV\bV^\top\|_\v\infty
+ \|\bU\|_{2\to\infty} \|\bV\|_{2\to\infty}
; \\
\gamma & \ := \ \|\orth(\bXL)\|_\v\infty \ = \ \|\bU\bV^\top\|_\v\infty
.
\end{align*}
The parameter $\rho$ is a balancing parameter to handle disparity between
row and column dimensions.
The quantity $\alpha(\rho)$ is the maximum number of non-zero entries
in any single row or column.
The quantities $\beta(\rho)$ and $\gamma$ measure the coherence of the
singular vectors of $\bXL$, that is, the alignment of the singular vectors
with the coordinate bases.
For instance, under the conditions of Proposition~\ref{proposition:alpha-beta-gamma},
we have (with $\rho=\sqrt{n/m}$)
\[
\alpha(\rho) \leq c_1\sqrt{mn} , \quad
\beta\left(\rho \right)
\ \leq \ \frac{3c_2\rank(\bXL)}{\sqrt{mn}}
\quad \text{and} \quad
\gamma \ \leq \ \frac{c_2\rank(\bXL)}{\sqrt{mn}}
 \]
for some constants $c_1$ and $c_2$.

\subsection{Operator norms of projection operators}

We show that under the condition $\inf_{\rho > 0} \alpha(\rho)\beta(\rho) <
1$, the pair $(\bXS,\bXL)$ is identifiable from its sum $\bXS +
\bXL$~(Theorem~\ref{theorem:identifiability}).
This is achieved by proving that the composition of projection operators
$\P_{\bO}$ and $\P_{\bT}$ is a contraction as per
Lemma~\ref{lemma:composition}, which in turn implies that $\bO \cap \bT =
\{0\}$.

The following two lemmas bound the projection operators $\P_{\bO}$ and
$\P_{\bT}$ in complementary norms.
\begin{lemma} \label{lemma:sparse-bound}
For any $M \in \R^{m\times n}$ and $p \in \{1,\infty\}$, we have
\[ \|\P_{\bO}(M)\|_{p\to p} \leq \|\sign(\bXS)\|_{p \to p} \|M\|_\v\infty . \]
This implies, for all $\rho > 0$,
\[ \|\P_{\bO}\|_{\v\infty\to\sharp(\rho)} \leq \alpha(\rho) . \]
\end{lemma}
\begin{proof}
Define $s(\XS) \in \{0,1\}^{m \times n}$ to be the entry-wise absolute
value of $\sign(\XS)$.
We have
\begin{align*}
\|\P_{\bO}(M)\|_{p \to p}
& = \max\{ \|\P_{\bO}(M)v\|_p : \|v\|_p \leq 1 \} \\
& \leq \|\P_\bO(M)\|_\v\infty \max\{ \|s(\P_{\bO}(M))v\|_p : \|v\|_p \leq 1 \} \\
& \leq \|M\|_\v\infty \max\{ \|s(\bXS)v\|_p : \|v\|_p \leq 1 \} \\
& = \|M\|_\v\infty \|\sign(\bXS)\|_{p \to p}
.
\end{align*}
The second part follows from the definitions of $\|\cdot\|_{\sharp(\rho)}$ and
$\alpha(\rho)$.
\end{proof}

\begin{lemma} \label{lemma:lowrank-bound}
For any $M \in \R^{m\times n}$, we have
\[ \|\P_{\bT}(M)\|_\v\infty \leq
\|\bU\bU^\top\|_\v\infty \|M\|_{1\to1}
+ \|\bV\bV^\top\|_\v\infty \|M\|_{\infty\to\infty}
+ \|\bU\|_{2\to\infty} \|\bV\|_{2\to\infty} \|M\|_{2\to2}
. \]
This implies, for all $\rho > 0$,
\[ \|\P_{\bT}\|_{\sharp(\rho)\to\v\infty} \leq \beta(\rho) . \]
\end{lemma}
\begin{proof}
We have $\|\P_{\bT}(M)\|_\v\infty = \|\bU\bU^\top M + M\bV\bV^\top -
\bU\bU^\top M \bV\bV^\top\|_\v\infty \leq \|\bU\bU^\top M\|_\v\infty +
\|M\bV\bV^\top\|_\v\infty + \|\bU\bU^\top M \bV\bV^\top\|_\v\infty$ by the
triangle inequality.
The bounds for each term now follow from the definitions:
\begin{align*}
\|\bU\bU^\top M\|_\v\infty
& = \max_i \|M^\top \bU\bU^\top e_i\|_\infty \\
& \leq \|M^\top\|_{\infty\to\infty} \max_i \|\bU\bU^\top e_i\|_\infty \\
& = \|M\|_{1\to1} \|\bU\bU^\top\|_\v\infty
;
\end{align*}
\begin{align*}
\|M\bV\bV^\top\|_\v\infty
& = \max_j \|M\bV\bV^\top e_j\|_\infty \\
& \leq \|M\|_{\infty\to\infty} \max_j \|\bV\bV^\top e_j\|_\infty \\
& = \|M\|_{\infty\to\infty} \|\bV\bV\|_\v\infty
;
\end{align*}
and
\begin{align*}
\|\bU\bU^\top M\bV\bV^\top\|_\v\infty
& = \max_{i,j} |e_i^\top \bU(\bU^\top M \bV) \bV^\top e_j| \\
& \leq \max_{i,j} \|\bU^\top e_i\|_2 \|\bU^\top M \bV\|_{2\to2} \|\bV^\top
e_j\|_2
\quad \text{(Cauchy-Schwarz)} \\
& \leq \|M\|_{2\to2} \|\bU\|_{2\to\infty} \|\bV\|_{2\to\infty} \\
& \leq \|M\|_{\sharp(\rho)} \|\bU\|_{2\to\infty} \|\bV\|_{2\to\infty}
\quad \text{(Lemma~\ref{lemma:spectral-bound})}
.
\end{align*}
The second part now follows the definition of $\beta(\rho)$.
\end{proof}

Now we show that the composition of $\P_{\bO}$ and $\P_{\bT}$ gives a
contraction under the certain norms and their duals.
\begin{lemma} \label{lemma:composition-bound}
For all $\rho>0$,
\begin{enumerate}
\item $\|\P_{\bO} \circ \P_{\bT}\|_{\sharp(\rho)\to\sharp(\rho)}
\leq \alpha(\rho) \beta(\rho)$;
\item $\|\P_{\bT} \circ \P_{\bO}\|_{\v\infty\to\v\infty} \leq \alpha(\rho)
\beta(\rho)$;
\end{enumerate}
\end{lemma}
\begin{proof}
Immediate from Lemma~\ref{lemma:sparse-bound} and
Lemma~\ref{lemma:lowrank-bound}.
\end{proof}

\begin{lemma} \label{lemma:dual-composition-bound}
For all $\rho>0$,
\begin{enumerate}
\item $\|\P_{\bT} \circ \P_{\bO}\|_{\flat(\rho)\to\flat(\rho)}
\leq \alpha(\rho) \beta(\rho)$;
\item $\|\P_{\bO} \circ \P_{\bT}\|_{\v1\to\v1}
\leq \alpha(\rho) \beta(\rho)$.
\end{enumerate}
\end{lemma}
\begin{proof}
First note that $(\P_{\bT} \circ \P_{\bO})^* = \P_{\bO}^* \circ \P_{\bT}^*
= \P_{\bO} \circ \P_{\bT}$ because $\P_{\bO}$ and $\P_{\bT}$ are
self-adjoint, and similarly $(\P_{\bO} \circ \P_{\bT})^* = \P_{\bT} \circ
\P_{\bO}$.
Now the claim follows by Proposition~\ref{proposition:operator-dual} and
Lemma~\ref{lemma:composition-bound}, using the facts that
$\|\cdot\|_{\flat(\rho)}$ is dual to $\|\cdot\|_{\sharp(\rho)}$ and that
$\|\cdot\|_\v1$ is dual to $\|\cdot\|_\v\infty$.
\end{proof}
Note that Lemma~\ref{lemma:composition} is encompassed by
Lemma~\ref{lemma:dual-composition-bound}.
Another consequence of these contraction properties is the following
uncertainty principle, analogous to one stated by~\citet{CSPW09},
which effectively states that a matrix $X$ cannot have both
$\|\sign(X)\|_{\sharp(\rho)}$ and $\|\orth(X)\|_\v\infty$ simultaneously
small.
\begin{theorem} \label{theorem:uncertainty}
If $X = \bXS = \bXL \neq 0$, then $\inf_{\rho > 0}
\alpha(\rho)\beta(\rho) \geq 1$.
\end{theorem}
\begin{proof}
Note that the non-zero element $X$ lives in $\bO \cap \bT$, so we get the
conclusion by the contrapositive of Theorem~\ref{theorem:identifiability}.
\end{proof}

\subsection{Dual certificate}

The incoherence properties allow us to construct an approximate dual
certificate $(Q_\bO, Q_\bT) \in \bO \times \bT$ that is central to the
analysis of the optimization problems~\eqref{eq:opt1} and~\eqref{eq:opt2}.

The certificate is constructed as the solution to the linear system
\[
\left\{ \begin{array}{rcr}
\P_{\bO}(Q_\bO + Q_\bT + \mu^{-1} E) & = & \lambda \sign(\bXS) \\
\P_{\bT}(Q_\bO + Q_\bT + \mu^{-1} E) & = & \orth(\bXL)
\end{array} \right.
\]
for some matrix $E \in \R^{m \times n}$; this can be equivalently written
as
\[
\left[ \begin{array}{cc}
\I & \P_{\bO} \\
\P_{\bT} & \I
\end{array} \right]
\left[ \begin{array}{c}
Q_\bO \\
Q_\bT \\
\end{array} \right]
=
\left[ \begin{array}{r}
\lambda \sign(\bXS) - \mu^{-1} \P_{\bO}(E) \\
\orth(\bXL) - \mu^{-1} \P_{\bT}(E)
\end{array} \right]
.
\]
We show the existence of the dual certificate $(Q_\bO,Q_\bT)$ under the
conditions~\eqref{eq:opt2-cond1}, \eqref{eq:opt2-cond2}, and
\eqref{eq:opt2-cond3} relative to an arbitrary matrix $E$.
Recall that the recovery guarantees for the constrained formulation
requires the conditions with $E = 0$, while the guarantees for the
regularized formulation takes $E = Y - (\bXS + \bXL)$.
\begin{theorem} \label{theorem:dual-cert}
Pick any $c > 1$, $\rho > 0$, and $E \in \R^{m \times n}$.
Let $\bk := |\supp(\bXS)|$ and $\br := \rank(\bXL)$.
Let 
\begin{align*}
\epsilon_{2\to2} & := \|E\|_{2\to2} \\
\epsilon_\v\infty & := \|E\|_{\v\infty} + \|\P_\bT(E)\|_{\v\infty} .
\end{align*}
If the following conditions hold:
\begin{align}
& \alpha(\rho) \beta(\rho) < 1
\label{eq:cond1} \\
&
\lambda \ \leq \
\frac{(1 - \alpha(\rho)\beta(\rho))(1-c \cdot \mu^{-1} \epsilon_{2\to2})
- c \cdot \alpha(\rho) \mu^{-1} \epsilon_\v\infty
- c \cdot \alpha(\rho) \gamma}
{c \cdot \alpha(\rho)}
\label{eq:cond2} \\
& \lambda \ \geq \ c \cdot \frac{\gamma
+ \mu^{-1} 
 (2 - \alpha(\rho) \beta(\rho)) \epsilon_\v\infty}
{1 - \alpha(\rho) \beta(\rho) - c \cdot \alpha(\rho) \beta(\rho)}
\ > \ 0
\label{eq:cond3}
\end{align}
(these are a restatement of~\eqref{eq:opt2-cond1}, \eqref{eq:opt2-cond2},
and \eqref{eq:opt2-cond3}),
then
\begin{align*}
Q_\bO & := (\I - \P_{\bO} \circ \P_{\bT})^{-1}
\left( \lambda \sign(\bXS)
- \P_{\bO}(\orth(\bXL))
- \mu^{-1} (\P_{\bO} \circ \P_{\bT^\perp})(E)
\right)
\in \bO
\quad \text{and}
\\
Q_\bT & := (\I - \P_{\bT} \circ \P_{\bO})^{-1}
\left( \orth(\bXL)
- \lambda \P_{\bT}(\sign(\bXS))
- \mu^{-1} (\P_{\bT} \circ \P_{\bO^\perp})(E)
\right)
\in \bT
\end{align*}
are well-defined and satisfy
\[
\begin{array}{rcr}
\P_{\bO}(Q_\bO + Q_\bT + \mu^{-1} E) & = & \lambda \sign(\bXS)\\
\P_{\bT}(Q_\bO + Q_\bT + \mu^{-1} E) & = & \orth(\bXL)
\end{array}
\]
and
\[
\begin{array}{lcl}
\|\P_{\bO^\perp} (Q_\bO + Q_\bT + \mu^{-1} E) \|_\v\infty & \leq & \lambda/c \\
\|\P_{\bT^\perp} (Q_\bO + Q_\bT + \mu^{-1} E)\|_{2\to2} & \leq & 1/c .
\end{array}
\]
Moreover,
\begin{align*}
\|Q_\bO \|_{2\to2}
& \leq
\frac{\alpha(\rho)}{1-\alpha(\rho)\beta(\rho)} \cdot
\left(
\lambda 
+  \gamma
+  \mu^{-1} \epsilon_\v\infty
\right) \\
\|Q_\bT \|_{2\to2}
& \leq
\frac{2\alpha(\rho)}{1-\alpha(\rho)\beta(\rho)} \cdot
\left(
\lambda 
+  \gamma
+  \mu^{-1} \epsilon_\v\infty
\right) 
+ 1 + 2\mu^{-1}\epsilon_{2\to2}
\\
\|Q_\bT\|_* & \leq 2\br \|Q_\bT\|_{2\to2}
\\
\|Q_\bT \|_{\v\infty} 
& \leq
\frac{1}{1-\alpha(\rho)\beta(\rho)} \cdot
\left(
\lambda 
+  \gamma
+  \mu^{-1} \epsilon_\v\infty
\right) \\
\|Q_\bO \|_{\v\infty} 
& \leq
\frac{2}{1-\alpha(\rho)\beta(\rho)} \cdot
\left(
\lambda 
+  \gamma
+  \mu^{-1} \epsilon_\v\infty
\right) 
\\
\|Q_\bO\|_\v1 & \leq \bk \|Q_\bO\|_\v\infty
\\
\|Q_\bO + Q_\bT\|_\v2^2
& \leq
\lambda \|Q_\bO\|_\v1
\left( 1 + \mu^{-1} \lambda^{-1} \epsilon_\v\infty \right)
+ \|Q_\bT\|_*
\left( 1 + 2\mu^{-1}\epsilon_{2\to2} \right)
.
\end{align*}
\end{theorem}
\begin{remark}
The dual certificate constitutes an approximate subgradient in the sense
that $Q_\bO + Q_\bT + \mu^{-1} E$ is a subgradient of both $\lambda
\|\XS\|_\v1$ at $\XS = \bXS$, and $\|\XL\|_*$ at $\XL = \bXL$.
\end{remark}
\begin{proof}
Under the condition~\eqref{eq:cond1}, we have $\alpha(\rho) \beta(\rho) <
1$, and therefore Lemma~\ref{lemma:composition-bound} and
Lemma~\ref{lemma:block-inverse} imply that the operators $\I - \P_{\bO}
\circ \P_{\bT}$ and $\I - \P_{\bT} \circ \P_{\bO}$ are invertible and
satisfy
\[
\max\left\{
\|(\I - \P_{\bO} \circ \P_{\bT})^{-1}\|_{\sharp(\rho)\to\sharp(\rho)}
, \
\|(\I - \P_{\bT} \circ \P_{\bO})^{-1}\|_{\v\infty\to\v\infty}
\right\}
\ \leq \ \frac1{1-\alpha(\rho) \beta(\rho)}
.
\]
Thus $Q_\bO$ and $Q_\bT$ are well-defined.
We can bound $\|Q_\bO\|_{2\to2}$ as
\begin{align*}
\|Q_\bO\|_{2\to2} 
& \leq \|Q_\bO\|_{\sharp(\rho)}
\quad \text{(Lemma~\ref{lemma:spectral-bound})} \\
& = \left\| (I - \P_{\bO} \circ \P_{\bT})^{-1}
\left( \lambda \sign(\bXS)
- \P_{\bO}(\orth(\bXL))
- \mu^{-1} (\P_{\bO} \circ \P_{\bT^\perp})(E) \right)
\right\|_{\sharp(\rho)}
\\
& \leq
\frac1{1-\alpha(\rho)\beta(\rho)} \cdot
\left\|
\lambda \sign(\bXS)
- \P_{\bO}(\orth(\bXL))
- \mu^{-1} (\P_{\bO} \circ \P_{\bT^\perp})(E)
\right\|_{\sharp(\rho)}
\\
& \leq
\frac1{1-\alpha(\rho)\beta(\rho)} \cdot
\left(
\lambda \|\sign(\bXS)\|_{\sharp(\rho)}
+ \|\P_{\bO}(\orth(\bXL))\|_{\sharp(\rho)}
+ \mu^{-1} \|(\P_{\bO} \circ \P_{\bT^\perp})(E)\|_{\sharp(\rho)}
\right)
\\
& \leq
\frac{\alpha(\rho)}{1-\alpha(\rho)\beta(\rho)} \cdot
\left(
\lambda 
+  \gamma
+  \mu^{-1} \|\P_{\bT^\perp}(E)\|_\v\infty
\right)
\quad \text{(Lemma~\ref{lemma:sparse-bound})}
\\
& \leq
\frac{\alpha(\rho)}{1-\alpha(\rho)\beta(\rho)} \cdot
\left(
\lambda 
+  \gamma
+  \mu^{-1} \epsilon_\v\infty
\right)
.
\end{align*}
Above, we have used the bound $\|\P_{\bT^\perp}(E)\|_\v\infty = \|E -
\P_{\bT}(E)\|_\v\infty \leq \epsilon_\v\infty$.
Therefore,
\begin{align*}
\|\P_{\bT^\perp}(Q_\bO + \mu^{-1} E)\|_{2\to2}
& \leq \|(I-\bU\bU^\top) Q_\bO (I-\bV\bV^\top)\|_{2\to2}
+ \mu^{-1} \|P_{\bT^\perp}(E)\|_{2\to2}
\\
& \leq \|Q_\bO\|_{2\to2} + \mu^{-1} \epsilon_{2\to2} \\
& \leq
\frac{\alpha(\rho)}{1-\alpha(\rho) \beta(\rho)} \cdot (\lambda +
\gamma +  \mu^{-1} \epsilon_\v\infty)
+ \mu^{-1} \epsilon_{2\to2}
.
\end{align*}
The condition~\eqref{eq:cond2} now implies that this quantity is at most
$1/c$.

Now we bound $\|Q_\bT\|_\v\infty$ as
\begin{align*}
\|Q_\bT\|_\v\infty
& =
\left\|
(\I - \P_{\bT} \circ \P_{\bO})^{-1}
\left(
\orth(\bXL)
- \lambda \P_{\bT}(\sign(\bXS)
- \mu^{-1} (\P_{\bT} \circ \P_{\bO^\perp})(E)
\right)
\right\|_\v\infty
\\
& \leq
\frac1{1-\alpha(\rho)\beta(\rho)} \cdot
\left\|
\orth(\bXL)
- \lambda \P_{\bT} (\sign(\bXS)
- \mu^{-1} (\P_{\bT} \circ \P_{\bO^\perp})(E)
\right\|_\v\infty
\\
& \leq
\frac1{1-\alpha(\rho)\beta(\rho)} \cdot
\left(
\|\orth(\bXL)\|_\v\infty
+ \lambda \|\P_{\bT}(\sign(\bXS))\|_\v\infty
+ \mu^{-1} \|(\P_{\bT} \circ \P_{\bO^\perp})(E)\|_\v\infty
\right)
\\
& \leq
\frac1{1-\alpha(\rho)\beta(\rho)} \cdot
\left(
\gamma
+ \lambda \alpha(\rho) \beta(\rho)
+ \mu^{-1} \epsilon_\v\infty 
\right)
\quad \text{(Lemma~\ref{lemma:composition-bound})}
.
\end{align*}
Above, we have used the bound $\|(\P_{\bT}\circ\P_{\bO^\perp})(E)\|_\v\infty =
\|\P_\bT(E) - (\P_\bT\circ\P_\bO)(E)\|_\v\infty \leq \|\P_\bT(E)\|_\v\infty
+ \alpha(\rho)\beta(\rho)\|E\|_\v\infty \leq \epsilon_\v\infty$.
Therefore,
\begin{align*}
\|\P_{\bO^\perp}(Q_\bT +\mu^{-1} E)\|_\v\infty
& \leq \|Q_\bT\|_\v\infty + \mu^{-1}\|\P_{\bO^\perp}(E)\|_\v\infty \\
& \leq \frac1{1-\alpha(\rho) \beta(\rho)} 
\cdot
\left(
\gamma
+ \lambda \alpha(\rho) \beta(\rho)
+ \mu^{-1} \epsilon_\v\infty
\right)
+ \mu^{-1} \epsilon_\v\infty
.
\end{align*}
The condition~\eqref{eq:cond3} now implies that this quantity is at most
$\lambda/c$.

We also have
\begin{align*}
\|Q_\bT \|_{2\to2}
& = \|\P_{\bT}(Q_\bO + \mu^{-1} E) - \orth(\bXL)\|_{2\to2} \\
& \leq
\frac{2\alpha(\rho)}{1-\alpha(\rho)\beta(\rho)} \cdot
\left(
\lambda 
+  \gamma
+  \mu^{-1} \epsilon_\v\infty
\right) 
+ 1 + 2\mu^{-1} \epsilon_{2\to2}
\end{align*}
since $\|\P_\bT(Q_\bO)\|_{2\to2}\leq 2 \|Q_\bO\|_{2\to2}$
and $\|\P_\bT(E)\|_{2\to2} \leq 2\epsilon_{2\to2}$
by Lemma~\ref{lemma:projections}, and
\begin{align*}
\|Q_\bO \|_{\v\infty} 
& = \|P_{\bO}(Q_\bT + \mu^{-1} E) - \lambda \sign(\bXS)\|_{\v\infty} \\
& \leq
\frac{1}{1-\alpha(\rho)\beta(\rho)} \cdot
\left(
\lambda 
+  \gamma
+  \mu^{-1} \epsilon_\v\infty
\right) + \lambda + \mu^{-1} \epsilon_\v\infty .
\end{align*}
The bounds on $\|Q_\bT\|_*$ and $\|Q_\bO\|_\v1$ follow from the facts that
$\rank(Q_\bT) \leq 2\br$ and $\|\supp(Q_\bO)\| \leq \bk$.
Finally,
\begin{align*}
\|Q_\bO + Q_\bT\|_{\v2}^2
& = \ang{Q_\bO,\P_\bO(Q_\bO + Q_\bT)} + \ang{Q_\bT,\P_\bT(Q_\bO + Q_\bT)} \\
& = \ang{Q_\bO, \lambda \P_{\bO}(\sign(\bXS)) - \mu^{-1} \P_{\bO}(E)}
+ \ang{Q_\bT,\P_{\bT}(\orth(\bXL)) - \mu^{-1} \P_{\bT}(E)} \\
& \leq \lambda \|Q_\bO\|_{\v1}
\left( 1 + \mu^{-1}\lambda^{-1} \|\P_\bO(E)\|_\v\infty \right)
+ \|Q_\bT\|_*
\left( 1 + \mu^{-1} \|\P_\bT(E)\|_{2\to2} \right)
\\
& \leq \lambda \|Q_\bO\|_{\v1}
\left( 1 + \mu^{-1}\lambda^{-1} \epsilon_\v\infty \right)
+ \|Q_\bT\|_*
\left( 1 + 2\mu^{-1} \epsilon_{2\to2} \right)
.
\qedhere
\end{align*}

\end{proof}

\section{Analysis of constrained formulation} \label{section:opt1}

Throughout this section, we fix a target decomposition $(\bXS,\bXL)$ that
satisfies the constraints of~\eqref{eq:opt1}, and let $(\hXS,\hXL)$ be the
optimal solution to~\eqref{eq:opt1}.
Let $\DS := \hXS - \bXS$ and $\DL := \hXL - \bXL$.
We show that under the conditions of Theorem~\ref{theorem:dual-cert} with
$E=0$ and appropriately chosen $\lambda$, solving \eqref{eq:opt1}
accurately recovers the target decomposition $(\bXS,\bXL)$.

We decompose the errors into symmetric and antisymmetric parts $\Da := (\DS
+ \DL)/2$ and $\Dm := (\DS - \DL) / 2$.
The constraints allow us to easily bound $\Da$, so most of the analysis
involves bounding $\Dm$ in terms of $\Da$.
\begin{lemma} \label{lemma:opt1-avg}
$\|\Da\|_\v1 \leq \epsilon_\v1$ and $\|\Da\|_* \leq \epsilon_*$.
\end{lemma}
\begin{proof}
Since both $(\hXS,\hXL)$ and $(\bXS,\bXL)$ as feasible solutions
to~\eqref{eq:opt1}, we have for $\diamondsuit \in \{\v1, *\}$,
\begin{align*}
\|\Da\|_\diamondsuit
& = \nicefrac12 \|\DS + \DL\|_\diamondsuit \\
& = \nicefrac12 \|(\hXS + \hXL - Y) - (\bXS + \bXL - Y)\|_\diamondsuit \\
& \leq \nicefrac12 \left( \|\hXS + \hXL - Y\|_\diamondsuit + \|\bXS + \bXL -
Y\|_\diamondsuit
\right) \\
& \leq \epsilon_\diamondsuit
.
\qedhere
\end{align*}
\end{proof}

\begin{lemma} \label{lemma:opt1-outside}
Assume the conditions of Theorem~\ref{theorem:dual-cert} hold with $E=0$.
We have
\[
\lambda \|\P_{\bO^\perp}(\Dm)\|_\v1 + \|\P_{\bT^\perp}(\Dm)\|_*
\leq (1-1/c)^{-1} \left( \lambda \|\Da\|_\v1 + \|\Da\|_* \right)
.
\]
\end{lemma}
\begin{proof}
Let $Q := Q_\bO + Q_\bT$ be the dual certificate guaranteed by
Theorem~\ref{theorem:dual-cert}.
Note that $Q$ satisfies the conditions of Lemma~\ref{lemma:subgradient}, so
we have
\begin{multline*}
\lambda \|\bXS + \Dm\|_\v1 + \|\bXL - \Dm\|_* - \lambda  \|\bXS\|_\v1 -
\|\bXL\|_*
\\
\geq (1-1/c) \left( \lambda \|\P_{\bO^\perp}(\Dm)\|_\v1 +
\|\P_{\bT^\perp}(\Dm)\|_* \right)
.
\end{multline*}
Using the triangle inequality, we have
\begin{align*}
\lambda \|\hXS\|_\v1 + \|\hXL\|_*
& = \lambda \|\bXS + \DS\|_\v1 + \|\bXL + \DL\|_* \\
& = \lambda \|\bXS + \Dm + \Da\|_\v1 + \|\bXL - \Dm + \Da\|_* \\
& \geq
\lambda \|\bXS + \Dm\|_\v1 - \lambda \|\Da\|_\v1
+ \|\bXL - \Dm\|_* - \|\Da\|_*
.
\end{align*}
Now using the fact that $\lambda \|\hXS\|_\v1 + \|\hXL\|_* \leq \lambda
\|\bXS\|_\v1 + \|\bXL\|_*$ gives the claim.
\end{proof}

\begin{lemma} \label{lemma:opt1-inside}
Let $\bk := |\supp(\bXS)|$.
Assume the conditions of Theorem~\ref{theorem:dual-cert} hold with $E = 0$.
We have
\[ \|\P_{\bO}(\Dm)\|_\v1
 \leq \frac{(1-1/c)^{-1}}{1-\alpha(\rho)\beta(\rho)}
\cdot (\|\Da\|_\v1 + \|\Da\|_* / \lambda)
.
\]
\end{lemma}
\begin{proof}
Because $\Dm = \P_{\bO}(\Dm) + \P_{\bO^\perp}(\Dm) = \P_{\bT}(\Dm) +
\P_{\bT^\perp}(\Dm)$, we have the equation
\[ \P_{\bO}(\Dm) - \P_{\bT}(\Dm) = - \P_{\bO^\perp}(\Dm) +
\P_{\bT^\perp}(\Dm)
. \]
Separately applying $\P_{\bO}$ and $\P_{\bT}$ to both sides gives
\[
\left[ \begin{array}{cc}
\I & \P_{\bO} \\
\P_{\bT} & \I
\end{array} \right]
\left[ \begin{array}{r}
\P_{\bO}(\Dm) \\
-\P_{\bT}(\Dm)
\end{array} \right]
=
\left[ \begin{array}{r}
(\P_{\bO} \circ \P_{\bT^\perp})(\Dm) \\
-(\P_{\bT} \circ \P_{\bO^\perp})(\Dm)
\end{array} \right]
.
\]
Under the condition $\alpha(\rho) \beta(\rho) < 1$,
Lemma~\ref{lemma:dual-composition-bound} and
Lemma~\ref{lemma:block-inverse} imply that
\[ \|(\I - \P_{\bO} \circ
\P_{\bT})^{-1}\|_{\v1\to\v1} \leq
\frac1{1-\alpha(\rho) \beta(\rho)}
\]
and that
\begin{align*}
\P_{\bO}(\Dm)
& = (\I - \P_{\bO} \circ \P_{\bT})^{-1}
\left(
(\P_{\bO} \circ \P_{\bT^\perp})(\Dm)
- (\P_{\bO} \circ \P_{\bT} \circ \P_{\bO^\perp})(\Dm)
\right)
.
\end{align*}
Therefore
\begin{align*}
\lefteqn{\|\P_{\bO}(\Dm)\|_\v1} \\
& \leq \frac1{1-\alpha(\rho)\beta(\rho)}
\cdot \left(
\|(\P_{\bO} \circ \P_{\bT^\perp})(\Dm)\|_\v1
+ \|(\P_{\bO} \circ \P_{\bT} \circ \P_{\bO^\perp})(\Dm)\|_\v1
\right)
\\
& \leq \frac1{1-\alpha(\rho)\beta(\rho)}
\cdot \left(
\sqrt{\bk} \cdot \|\P_{\bT^\perp}(\Dm)\|_\v2
+ \alpha(\rho) \beta(\rho) \cdot \|\P_{\bO^\perp}(\Dm)\|_\v1
\right)
\quad \text{(Lemma~\ref{lemma:dual-composition-bound})}
\\
& \leq \frac1{1-\alpha(\rho)\beta(\rho)}
\cdot \left(
\sqrt{\bk} \cdot \|\P_{\bT^\perp}(\Dm)\|_*
+ \alpha(\rho) \beta(\rho) \cdot \|\P_{\bO^\perp}(\Dm)\|_\v1
\right)
\\
& \leq \frac{(1-1/c)^{-1}}{1-\alpha(\rho)\beta(\rho)}
\cdot \max\left\{ \sqrt{\bk}, \ \alpha(\rho)\beta(\rho) / \lambda \right\}
\cdot (\lambda \|\Da\|_\v1 + \|\Da\|_*)
\quad \text{(Lemma~\ref{lemma:opt1-outside})}
\\
& \leq \frac{(1-1/c)^{-1}}{1-\alpha(\rho)\beta(\rho)}
\cdot (\|\Da\|_\v1 + \|\Da\|_* / \lambda)
\end{align*}
where the last inequality uses the facts $\bk \leq \alpha(\rho)^2$,
$\alpha(\rho)\beta(\rho) < 1$, and $\lambda \alpha(\rho) \leq 1$.
\end{proof}

We now prove Theorem~\ref{theorem:opt1}, which we restate here for
convenience.
\begin{theorem}[Theorem~\ref{theorem:opt1} restated] \label{theorem:opt1-again}
Assume the conditions of Theorem~\ref{theorem:dual-cert} hold with $E = 0$.
We have
\begin{multline*}
\max\{ \|\DS\|_\v1, \ \|\DL\|_\v1 \}
\\
\leq
\left(
1 +
(1-1/c)^{-1} \cdot
\frac{2-\alpha(\rho)\beta(\rho)}{1-\alpha(\rho)\beta(\rho)}
\right)
\cdot \epsilon_\v1
+ (1-1/c)^{-1} \cdot
\frac{2-\alpha(\rho)\beta(\rho)}{1-\alpha(\rho)\beta(\rho)}
\cdot
\epsilon_* / \lambda
.
\end{multline*}
If, in addition for some $b \geq \|\bXL\|_\v\infty$, either:
\begin{itemize}
\item the optimization problem~\eqref{eq:opt1} is augmented with the
constraint $\|\XL\|_\v\infty \leq b$, or

\item $\hXL$ is post-processed by replacing $[\hXL]_{i,j}$ with
$\min\{ \max\{ [\hXL]_{i,j}, -b\}, b \}$ for all $i,j$,

\end{itemize}
then we also have
\[ \|\DL\|_\v2 \leq \min\left\{ \|\DL\|_\v1, \ \sqrt{2b \|\DL\|_\v1}
\right\} . \]
\end{theorem}
\begin{proof}
First note that since $\DS = \Da + \Dm$ and $\DL = \Da - \Dm$, we have
$\max\{\|\DS\|_\v1,\|\DL\|_\v1\} \leq \|\Da\|_\v1 + \|\Dm\|_\v1$.
We can bound $\|\Dm\|_\v1$ as
\begin{align*}
\|\Dm\|_\v1
& \leq \|\P_{\bO^\perp}(\Dm)\|_\v1 + \|\P_{\bO}(\Dm)\|_\v1 \\
& \leq
(1-1/c)^{-1} \cdot \left(
1 + \frac{1}{1-\alpha(\rho)\beta(\rho)}
\right)
\cdot ( \|\Da\|_\v1 + \|\Da\|_* / \lambda)
\end{align*}
by Lemma~\ref{lemma:opt1-outside} and Lemma~\ref{lemma:opt1-inside}.
The bounds on $\|\DS\|_\v1$ and $\|\DL\|_\v1$ follow from the bounds on
$\|\Dm\|_\v1$, $\|\Da\|_\v1$, and $\|\Da\|_*$ (from
Lemma~\ref{lemma:opt1-avg}).

If the constraint $\|\XL\|_\v\infty \leq b$ is added, then we can use the
facts $\|\DL\|_\v\infty \leq \|\hXL\|_\v\infty + \|\bXL\|_\v\infty \leq 2b$
and $\|\DL\|_\v2 \sqrt{\|\DL\|_\v\infty \|\DL\|_\v1} \leq
\sqrt{2b\|\DL\|_\v1}$.
If $\hXL$ is post-processed, then (letting $\clip(\hXL)$ be the result of
the post-processing) $|\clip(\hXL)_{i,j} - [\bXL]_{i,j}| \leq |[\hXL]_{i,j}
- [\bXL]_{i,j}|$ for all $i,j$, so $\|\clip(\hXL) - \bXL\|_\v1 \leq
\|\DL\|_\v1$ and $\|\clip(\hXL)-\bXL\|_\v2 \leq \sqrt{2b
\|\clip(\hXL)-\bXL\|_\v1}$.
\end{proof}

\section{Analysis of regularized formulation} \label{section:opt2}

Throughout this section, we fix a target decomposition $(\bXS,\bXL)$ that
satisfies $\|\bXS-Y\|_\v\infty \leq b$, and let $(\hXS, \hXL)$ be the
optimal solution to~\eqref{eq:opt2} augmented with the constraint
$\|\XS-Y\|_\v\infty \leq b$ for some $b \geq \|\bXS-Y\|_\v\infty$ ($b =
\infty$ is allowed).
Let $\DS := \hXS-\bXS$ and $\DL := \hXL-\bXL$.
We show that under the conditions of Theorem~\ref{theorem:dual-cert} with
$E=Y-(\bXS+\bXL)$ and appropriately chosen $\lambda$ and $\mu$, solving
\eqref{eq:opt2} accurately recovers the target decomposition $(\bXS,\bXL)$.

\begin{lemma} \label{lemma:opt2-kkt}
There exists $G_S, G_L, H_S \in \R^{m \times n}$ such that
\begin{enumerate}
\item $\mu^{-1} ( \hXS + \hXL - Y ) + \lambda G_S + H = 0$;
$\|G_S\|_\v\infty \leq 1$;

\item $\mu^{-1} ( \hXS + \hXL - Y ) + \lambda G_L = 0$;
$\|G_L\|_{2\to2} \leq 1$;

\item $[H_S]_{i,j} [\DS]_{i,j} \geq 0 \ \forall i,j$ .

\end{enumerate}
\end{lemma}
\begin{proof}
We express the constraint $\|\XS-Y\|_\v\infty \leq b$ in~\eqref{eq:opt2} as
$2mn$ constraints $[\XS]_{i,j} - Y_{i,j} - b \leq 0$ and $-[\XS]_{i,j} +
Y_{i,j} - b \leq 0$ for all $i,j$.
Now the corresponding Lagrangian is
\[
\frac1{2\mu} \|\XS+\XL-Y\|_\v2^2 + \lambda \|\XS\|_\v1 + \|\XL\|_*
+ \ang{\Lambda^+,\XS-Y-b 1_{m,n}}
+ \ang{\Lambda^-,-\XS+Y-b 1_{m,n}}
\]
where $\Lambda^+, \Lambda^- \geq 0$ and $1_{m,n}$ is the all-ones $m \times
n$ matrix.
First-order optimality conditions imply that there exists a subgradient
$G_S$ of $\|\XS\|_\v1$ at $\XS=\hXS$ and a subgradient $G_L$ of $\|\XL\|_*$
at $\XL=\hXL$ such that
\[
\mu^{-1} (\hXS + \hXL - Y) + \lambda G_S + (\Lambda^+ - \Lambda^-) = 0
\quad \text{and} \quad
\mu^{-1} (\hXS + \hXL - Y) + G_L = 0
.
\]
Now since $\|\bXS-Y\|_\v\infty \leq b$, we have
$[\bXS]_{i,j} \leq Y_{i,j} + b$
and
$-[\bXS]_{i,j} \leq -Y_{i,j} + b$.
By complementary slackness, if $\Lambda^+_{i,j} > 0$, then $[\hXS]_{i,j} -
Y_{i,j} - b = 0$, which means $[\hXS]_{i,j} - [\bXS]_{i,j} \geq
[\hXS]_{i,j} - (Y_{i,j} + b) = 0$.
So $\Lambda^+_{i,j} [\DS]_{i,j} \geq 0$.
Similarly, if $\Lambda^-_{i,j} > 0$, then $[\hXS]_{i,j} - [\bXS]_{i,j} \leq
0$.
So $\Lambda^-_{i,j} [\DS]_{i,j} \leq 0$.
Therefore $H := \Lambda^+ - \Lambda^-$ satisfies $H_{i,j} [\DS]_{i,j} \geq
0$.
\end{proof}

\begin{lemma} \label{lemma:opt2-outside}
Assume the conditions of Theorem~\ref{theorem:dual-cert} hold with $E = Y -
(\bXS + \bXL)$, and let $(Q_\bO,Q_\bT)$ be the dual certificate from the
conclusion.
We have
\[
\lambda \|\P_{\bO^\perp} (\DS)\|_{\v1} + \|\P_{\bT^\perp} (\DL)\|_* 
\leq (1-1/c)^{-1} \|Q_\bO + Q_\bT\|_\v2^2 \mu/2
.
\]
\end{lemma}
\begin{proof}
Let $Q := Q_\bO + Q_\bT$ and $\Delta := \DS + \DL$.
Since $Q + \mu^{-1} E$ satisfies the conditions of
Lemma~\ref{lemma:subgradient},
\begin{multline*}
(1-1/c) \left( \lambda \|\P_{\bO^\perp} (\DS)\|_{\v1} + \|\P_{\bT^\perp}
(\DL)\|_* \right)
\\
\leq
( \lambda \|\hXS\|_\v1 + \|\hXL\|_* )
- ( \lambda \|\bXS\|_\v1 + \|\bXL\|_* )
- \ang{Q + \mu^{-1} E, \DS+\DL}
.
\end{multline*}
Furthermore, by the optimality of $(\hXS,\hXL)$,
\begin{align*} 
( \lambda \|\hXS\|_\v1 + \|\hXL\|_* )
- ( \lambda \|\bXS\|_\v1 + \|\bXL\|_* )
& \leq
\frac1{2\mu} \|\bXS + \bXL - Y\|_\v2^2
- \frac1{2\mu} \|\hXS + \hXL - Y\|_\v2^2
\\
& =
\frac1{2\mu} \|E\|_\v2^2
- \frac1{2\mu} \|\DS+\DL-E\|_\v2^2
\\
& =
\frac{1}{2\mu} (2\ang{E,\Delta} - \ang{\Delta,\Delta})
.
\end{align*} 
Combining the inequalities gives
\begin{align*}
(1-1/c) \left( \lambda \|\P_{\bO^\perp} (\DS)\|_{\v1} + \|\P_{\bT^\perp}
(\DL)\|_* \right)
& \leq -\ang{Q,\Delta} -\frac1{2\mu} \ang{\Delta,\Delta}
 \leq \|Q\|_\v2^2 \mu/2
\end{align*}
where the last inequality follows by taking the maximum value over $\Delta$
at $\Delta = -\mu Q$.
\end{proof}

Now we prove Theorem~\ref{theorem:opt2}, restated below (with an additional result for $\|\Delta_L\|_{\flat(\rho)}$).
\begin{theorem}[Theorem~\ref{theorem:opt2} restated] \label{theorem:opt2-again}
Let $\bk := |\supp(\bXS)|$ and $\br := \rank(\bXL)$.
Assume the conditions of Theorem~\ref{theorem:dual-cert} hold with $E = Y -
(\bXS + \bXL)$, and let $(Q_\bO,Q_\bT)$ be the dual certificate from the
conclusion.
We have
\begin{align*}
\|\DS\|_\v1
& \leq
\frac
{
\lambda^{-1}
(1-1/c)^{-1}
\|Q_\bO+Q_\bT\|_\v2^2
\mu
+ \lambda \bk \mu
+ 2\sqrt{\bk\br} \mu
+ \bk \|(\P_{\bO} \circ \P_{\bT^\perp})(E)\|_\v\infty
}
{1-\alpha(\rho)\beta(\rho)}
\\
\|\DS\|_\v2 & \leq \min\left\{ \|\DS\|_\v1, \ \sqrt{2b\|\DS\|_\v1} \right\}
\\
\|\DL\|_{\flat(\rho)}
& \leq (1-1/c)^{-1} \|Q_\bO+Q_\bT\|_\v2^2 \mu / 2
+ \min\left\{ \beta(\rho) \|\DS\|_\v1, \ \sqrt{2\br} \|\DS\|_\v2 \right\}
+ \|\P_\bT(E)\|_* + 2\br\mu
\\
\|\DL\|_*
& \leq (1-1/c)^{-1} \|Q_\bO+Q_\bT\|_\v2^2 \mu / 2
+ \sqrt{2\br} \|\DS\|_\v2 + \|\P_\bT(E)\|_* + 2\br\mu
.
\end{align*}
\end{theorem}
\begin{proof}
From Lemma~\ref{lemma:opt2-kkt}, we obtain $G_S, G_L, H_S \in \R^{m \times
n}$ and the following equations:
\begin{align}
& \mu^{-1} (\P_\bO(\DS) + \P_\bO(\DL) - \P_\bO(E)) + \P_\bO(H_S)
= -\lambda \P_\bO(G_S)
\label{eq:kkt1} \\
& \mu^{-1} (\P_\bT(\DS) + \P_\bT(\DL) - \P_\bT(E)) = -\P_\bT(G_L)
\label{eq:kkt2} \\
& \mu^{-1} ((\P_\bO\circ\P_\bT)(\DS) + (\P_\bO\circ\P_\bT)(\DL)
- (\P_\bO\circ\P_\bT)(E))
= -(\P_\bO\circ\P_\bT)(G_L)
\label{eq:kkt3}
.
\end{align}
Subtracting~\eqref{eq:kkt3} from~\eqref{eq:kkt1} gives
\begin{multline*}
\mu^{-1} (
\P_{\bO}(\DS)
- (\P_{\bO} \circ \P_{\bT} \circ \P_{\bO})(\DS)
- (\P_{\bO} \circ \P_{\bT} \circ \P_{\bO^\perp})(\DS)
- (\P_{\bO} \circ \P_{\bT^\perp})(\DL)
)
+ \P_{\bO}(H_S)
\\
 =
-\lambda \P_{\bO}(G_S)
+ (\P_{\bO} \circ \P_{\bT})(G_L)
+ \mu^{-1} (\P_{\bO} \circ \P_{\bT^\perp})(E)
.
\end{multline*}
Moreover, we have
$\ang{\sign(\DS),\P_{\bO}(\DS)} = \|\P_{\bO}(\DS)\|_\v1$
and
$\ang{\sign(\DS),\P_{\bO}(H_S)} = \|\P_{\bO}(H_S)\|_\v1$,
so taking inner products with $\sign(\DS)$ on both sides of the equation
gives
\begin{align*}
\lefteqn{
\mu^{-1} \|\P_{\bO}(\DS)\|_\v1 + \|\P_{\bO}(H_S)\|_\v1
} \\
& \leq
\mu^{-1} \|(\P_{\bO}\circ\P_{\bT}\circ\P_{\bO})(\DS)\|_\v1
+ \mu^{-1} \|(\P_{\bO} \circ \P_{\bT} \circ \P_{\bO^\perp})(\DS)\|_\v1
+ \mu^{-1} \|(\P_{\bO} \circ \P_{\bT^\perp})(\DL)\|_\v1
\\
& \qquad{}
+ \lambda \|\P_{\bO}(G_S)\|_\v1
+ \|(\P_{\bO} \circ \P_{\bT})(G_L)\|_\v1
+ \mu^{-1} \|(\P_{\bO} \circ \P_{\bT^\perp})(E)\|_\v1
\\
& \leq
\mu^{-1} \alpha(\rho)\beta(\rho) \|\P_{\bO}(\DS)\|_\v1
+ \mu^{-1} \alpha(\rho)\beta(\rho) \|\P_{\bO^\perp}(\DS)\|_\v1
+ \mu^{-1} \sqrt{\bk} \|\P_{\bT^\perp}(\DL)\|_\v2
\\
& \qquad{}
+ \lambda \bk
+ \sqrt{\bk} \|\P_{\bT}(G_L)\|_\v2
+ \mu^{-1} \bk \|(\P_{\bO} \circ \P_{\bT^\perp})(E)\|_\v\infty
\\
& \leq
\mu^{-1} \alpha(\rho)\beta(\rho) \|\P_{\bO}(\DS)\|_\v1
+ \mu^{-1} \alpha(\rho)\beta(\rho) \|\P_{\bO^\perp}(\DS)\|_\v1
+ \mu^{-1} \sqrt{\bk} \|\P_{\bT^\perp}(\DL)\|_\v2
\\
& \qquad{}
+ \lambda \bk
+ 2\sqrt{\bk\br} \|G_L\|_{2\to2}
+ \mu^{-1} \bk \|(\P_{\bO} \circ \P_{\bT^\perp})(E)\|_\v\infty
\\
& \leq
\mu^{-1} \alpha(\rho)\beta(\rho) \|\P_{\bO}(\DS)\|_\v1
+ \mu^{-1} \alpha(\rho)\beta(\rho) \|\P_{\bO^\perp}(\DS)\|_\v1
+ \mu^{-1} \sqrt{\bk} \|\P_{\bT^\perp}(\DL)\|_\v2
\\
& \qquad{}
+ \lambda \bk
+ 2\sqrt{\bk\br}
+ \mu^{-1} \bk \|(\P_{\bO} \circ \P_{\bT^\perp})(E)\|_\v\infty
.
\end{align*}
The second and third inequalities above follow from
Lemma~\ref{lemma:projections}
and
Lemma~\ref{lemma:dual-composition-bound},
and the fourth inequality uses the fact that $\|G_L\|_{2\to2} \leq 1$.
Rearranging the inequality and applying Lemma~\ref{lemma:opt2-outside}
gives
\begin{align*}
\lefteqn{
(1-\alpha(\rho)\beta(\rho)) \|\P_{\bO}(\DS)\|_\v1
} \\
& \leq
\alpha(\rho)\beta(\rho) \|\P_{\bO^\perp}(\DS)\|_\v1
+ \sqrt{\bk} \|\P_{\bT^\perp}(\DL)\|_\v2
+ \lambda \bk \mu
+ 2\sqrt{\bk\br} \mu
+ \bk \|(\P_{\bO} \circ \P_{\bT^\perp})(E)\|_\v\infty
\\
& \leq
\max\{\alpha(\rho)\beta(\rho)/\lambda,\ \sqrt{\bk}\}
(1-1/c)^{-1}
\|Q_\bO+Q_\bT\|_\v2^2
\mu / 2
+ \lambda \bk \mu
+ 2\sqrt{\bk\br} \mu
+ \bk \|(\P_{\bO} \circ \P_{\bT^\perp})(E)\|_\v\infty
\\
& \leq
\lambda^{-1}
(1-1/c)^{-1}
\|Q_\bO+Q_\bT\|_\v2^2
\mu / 2
+ \lambda \bk \mu
+ 2\sqrt{\bk\br} \mu
+ \bk \|(\P_{\bO} \circ \P_{\bT^\perp})(E)\|_\v\infty
\end{align*}
since $\bk \leq \alpha(\rho)^2$, $\alpha(\rho)\beta(\rho) < 1$, and
$\lambda \alpha(\rho) \leq 1$.
Now we combine this with $\|\DS\|_\v1 \leq \|\P_{\bO^\perp}(\DS)\|_\v1 +
\|\P_{\bO}(\DS)\|_\v1$ and Lemma~\ref{lemma:opt2-outside} to get the first
bound.

For the second bound, we use the facts
$\|\DS\|_\v\infty \leq \|\hXS-Y\|_\v\infty + \|\bXS-Y\|_\v\infty \leq 2b$
and
$\|\DS\|_\v2 \leq \sqrt{\|\DS\|_\v1\|\DS\|_\v\infty} \leq
\sqrt{2b\|\DS\|_\v1}$.

For the third and fourth bounds, we obtain from~\eqref{eq:kkt2}
\begin{align*}
\|\P_{\bT}(\DL)\|_{\flat(\rho)}
& \leq \|\P_\bT(\DS)\|_{\flat(\rho)}
+ \|\P_\bT(E)\|_{\flat(\rho)}
+ \mu\|\P_\bT(G_L)\|_{\flat(\rho)}
\\
& \leq \|\P_\bT\|_{\v1\to\flat(\rho)} \|\DS\|_\v1
+ \|\P_\bT(E)\|_*
+ \mu\|\P_\bT(G_L)\|_*
\quad \text{(Lemma~\ref{lemma:trace-bound})}
\\
& = \|\P_\bT^*\|_{\sharp(\rho)\to\v\infty} \|\DS\|_\v1
+ \|\P_\bT(E)\|_*
+ \mu\|\P_\bT(G_L)\|_*
\quad \text{(Proposition~\ref{proposition:operator-dual})}
\\
& \leq \beta(\rho) \|\DS\|_\v1
+ \|\P_\bT(E)\|_*
+ \mu\|\P_\bT(G_L)\|_*
\quad \text{(Lemma~\ref{lemma:lowrank-bound})}
\\
& \leq \beta(\rho) \|\DS\|_\v1
+ \|\P_\bT(E)\|_*
+ 2\br\mu
\quad \text{(Lemma~\ref{lemma:projections} and $\|G_L\|_{2\to2} \leq 1$)}
\end{align*}
and
\begin{align*}
\|\P_{\bT}(\DL)\|_*
& \leq \|\P_\bT(\DS)\|_*
+ \|\P_\bT(E)\|_*
+ \mu\|\P_\bT(G_L)\|_*
\\
& \leq \sqrt{2\br} \|\DS\|_\v2
+ \|\P_\bT(E)\|_*
+ 2\br\mu
\quad \text{(Lemma~\ref{lemma:projections} and $\|G_L\|_{2\to2} \leq 1$)}
.
\end{align*}
Now we combine these with
\begin{align*}
\|\DL\|_{\flat(\rho)}
& \leq \|\P_{\bT^\perp}(\DL)\|_{\flat(\rho)} +
\|\P_{\bT}(\DL)\|_{\flat(\rho)}
\\
& \leq \|\P_{\bT^\perp}(\DL)\|_* +
\min\{\|\P_{\bT}(\DL)\|_*,\|\P_{\bT}(\DL)\|_{\flat(\rho)}\}
\quad \text{(Lemma~\ref{lemma:trace-bound})}
\\
\|\DL\|_*
& \leq \|\P_{\bT^\perp}(\DL)\|_* + \|\P_{\bT}(\DL)\|_*
\end{align*}
and Lemma~\ref{lemma:opt2-outside}.
\end{proof}
Note that we have an error bound for $\DL$ in $\|\cdot\|_{\flat(\rho)}$
norm, which can be significantly smaller than the bound for the trace norm
of $\DL$.

\subsubsection*{Acknowledgements}

We thank Emmanuel Cand\`es for clarifications about the results
in~\citep{CLMW09}.

\bibliography{sparse-lowrank}
\bibliographystyle{plainnat}

\end{document}